\documentclass{article}
\usepackage{graphicx} % Required for inserting images
\usepackage{float} 
\usepackage{url}
\usepackage{amsmath}
\usepackage{setspace}
\usepackage{parskip}
\usepackage{microtype}
\usepackage[margin=3cm]{geometry}
\usepackage{titlesec}
\usepackage{lmodern}
\usepackage{amsmath, amssymb}
\usepackage{fancyhdr}
\usepackage{natbib}
\usepackage{hyperref}
\usepackage{caption}
\usepackage{algorithm}% http://ctan.org/pkg/algorithms
\usepackage{algorithmic}
\usepackage{amsthm}
\usepackage{booktabs}
\usepackage{rotating} 
\usepackage{dsfont}
\usepackage{booktabs}
\usepackage{tabularx}
\usepackage{booktabs}
\usepackage{makecell}
\usepackage{array}
\usepackage{amsthm}
\usepackage{authblk}
\newtheorem{theorem}{Theorem}
\newtheorem{definition}{Definition}
\newtheorem*{theorem*}{Theorem}
\usepackage{xcolor}
\newcommand{\mbf}{\mathbf}
\newcommand{\bsym}{\boldsymbol}

\title{Concept frustration: Aligning human concepts and machine representations}

\author[1,2]{Enrico Parisini}
\author[3]{Christopher J. Soelistyo}
\author[4]{Ahab Isaac}
\author[4,5]{Alessandro Barp}
\author[1,5,*]{Christopher R.S. Banerji}

\affil[1]{Comprehensive Cancer Centre, Guy's Hospital, King’s College London, Great Maze Pond, London, SE1 9RT, UK}
\affil[2]{PharosAI, King’s College London, London, SE1 9RT, UK}
\affil[3]{The Francis Crick Institute, 1 Midland Rd, London, NW1 1AT, UK}
\affil[4]{Department of Statistical Science, University College London, WC1E 7HB, UK}
\affil[5]{The Alan Turing Institute, NW1 2DB, UK}

\affil[*]{Corresponding author: christopher.banerji@kcl.ac.uk}
\date{March 2026}

\begin{document}

\maketitle
%TC:ignore
\begin{abstract}
    \noindent
    Aligning human-interpretable concepts with the internal representations learned by modern machine learning systems remains a central challenge for interpretable AI. We introduce a geometric framework for comparing supervised human concepts with unsupervised intermediate representations extracted from foundation model embeddings. Motivated by the role of conceptual leaps in scientific discovery, we formalise the notion of concept frustration: a contradiction that arises when an unobserved concept induces relationships between known concepts that cannot be made consistent within an existing ontology. We develop task-aligned similarity measures that detect concept frustration between supervised concept-based models and unsupervised representations derived from foundation models, and show that the phenomenon is detectable in task-aligned geometry while conventional Euclidean comparisons fail. Under a linear–Gaussian generative model we derive a closed-form expression for Bayes-optimal concept-based classifier accuracy, decomposing predictive signal into known–known, known–unknown and unknown–unknown contributions and identifying analytically where frustration affects performance. Experiments on synthetic data and real language and vision tasks demonstrate that frustration can be detected in foundation model representations and that incorporating a frustrating concept into an interpretable model reorganises the geometry of learned concept representations, to better align human and machine reasoning. These results suggest a principled framework for diagnosing incomplete concept ontologies and aligning human and machine conceptual reasoning, with implications for the development and validation of safe interpretable AI for high-risk applications.
\end{abstract}
%TC:endignore
\section{Introduction}

Human-level interpretability in artificial intelligence (AI) is critical for high-consequence domains such as medicine, scientific discovery and criminal justice \cite{rudin2019stop,doshi2017towards,hassija2024interpreting}. In such settings, errors can have catastrophic consequences and accountability, as well as emerging regulation, requires not only accurate predictions but transparent reasoning \cite{EUMemberStates2024,Binns2021}. %Moreover, optimal decisions in high-stakes domains frequently involve subjective judgments. In medicine, for example, clinicians and patients may legitimately disagree about treatment strategies \cite{symmons2023decision,chandwani2017lack}. Without transparent reasoning, AI systems risk undermining informed consent and autonomy. 

There are two prominent strategies for securing transparency. (1) Post-hoc feature attribution, where we train a black box model, then perform techniques such as SHAP \cite{lundberg_unified_2017}, LIME \cite{ribeiro_why_2016} or Integrated Gradients \cite{sundararajan_axiomatic_2017, Saranya2023} to identify task-relevant elements of the input. Similarly, mechanistic interpretability decomposes feature representations of trained models using methods such as sparse autoencoders (SAEs) \cite{Cunningham2023,Elhage2022,bau2017network}. (2) Interpretability-by-design, where we optimise aspects of the model architecture, training regime or input set, for interpretability pre-hoc. This includes the concept bottleneck model (CBM), where we constrain the model to learn a set of supervised concepts that are pre-defined by human domain experts and therefore guaranteed to be interpretable \cite{Koh2020}.

Both strategies have limitations. Since post-hoc methods construct local surrogates of a complex black-box model, there is no guarantee they will provide accurate or holistic interpretations of model behaviour \cite{rudin2019stop}, conversely SAE derived features often exhibit polysemanticity and feature superposition, degrading interpretability \cite{Cunningham2023,Elhage2022,bau2017network}. In CBMs, the bottleneck constraint that ensures interpretability also limits model expressivity, undermining performance when task-relevant features are excluded from the supervised concepts \cite{Zarlenga2022,Debot2025}.

There are at least two solutions to this problem in CBMs. The first is to allow predictive information to flow through an auxiliary pathway that bypasses the concept bottleneck \cite{Mahinpei2021PromisesAP, Havasi2022, Sawada2022ConceptBM, yuksekgonul2023posthoc, ismail2024concept}. These ``concept side-channel models'' increase expressivity but at the expense of interpretability. The second is to optimise the concept set - or \emph{ontology} - for the task at hand. However, constructing an adequate ontology first requires a reliable diagnostic for cases when the current ontology is inadequate. Recent work has addressed this by comparing supervised CBM concepts with unsupervised SAE features through geometric alignment in Euclidean space \cite{Rocchi--Henry2025}. While foundational, such approaches rely on restrictive assumptions: that supervised concepts are statistically independent, that alignment should be assessed in a task-agnostic geometry and that particular parameter constraints (such as non-negativity of concept codes) are preserved. These assumptions are often violated in practice.

%In this study, we provide a novel diagnostic framework for incomplete ontologies that addresses the limitations of prior work. Concretely, we identify a phenomenon that we term \emph{concept frustration}, where there exists task-relevant variation in the input space that cannot be coherently captured in concept space. We provide a formal definition for frustration, and demonstrate that detecting frustration requires working in a task-aligned geometry, rather than a Euclidean geometry. Finally, we empirically show in real-world language and vision tasks that frustration degrades performance and interpretability in CBMs. These results suggest that concept frustration is vital to understanding whether ontologies are incomplete, and how this could be rectified.

Here we introduce a diagnostic framework for incomplete concept ontologies that addresses limitations of prior work. We identify a phenomenon we term \emph{concept frustration}, in which information that is predictive of the task but absent from the supervised concept set induces contradictions between known concepts, revealing that the ontology is incomplete. We show that detecting frustration requires a task-aligned geometry, rather than a Euclidean geometry, and that frustration degrades both performance and interpretability in CBMs. In real-world language and vision tasks, resolving frustration by incorporating a previously unobserved frustrating concept reorganises existing concept geometry, restoring consistency between supervised and unsupervised concept sets. 

Together, these results provide a pathway for identifying hidden conceptual dimensions, that can fundamentally reorganise how supervised models, and even human overseers, reason about the world.

\section{Results}

\subsection{Flatworld: a motivating example}

To illustrate frustration, consider a thought experiment: a flatworlder navigates the Earth using two trusted concepts: distance to the North Pole and distance to the South Pole. They expect these concepts to oppose each other, so that moving closer to one Pole moves them farther from the other. However, due to the Earth's curvature, a traveller who moves closer to the Earth's centre, say by diving in an ocean, can move closer to both Poles simultaneously (Figure \ref{fig:frustration}). 

This introduces an unknown concept, depth, that is not part of the flatworlder’s ontology but distorts how known concepts relate. We refer to such a hidden, task-relevant concept as \emph{frustrating}: without it, the flatworlder cannot form a coherent picture of their world.

%\subsection{Flat-world: a motivating example}

% Consider a thought experiment: a flat-worlder is navigating the surface of the Earth using two trusted concepts: distance to the North Pole and distance to the South Pole. They expect that these concepts oppose each other, so that moving closer to one Pole moves them farther from the other. However, due to the Earth's curvature, a traveller who gets closer to the centre of the Earth, say by diving in an ocean, can move in a way that gets them closer to both Poles simultaneously (Figure \ref{fig:frustration}). This concept of `depth' is unknown to the flat-worlder and can only be observed through its distortion of how known concepts interrelate. It is therefore a \emph{frustrating} concept, without which the flat-worlder cannot fully understand their world.

%TC:ignore
\begin{figure}[H]
    \centering
    \includegraphics[width=\linewidth]{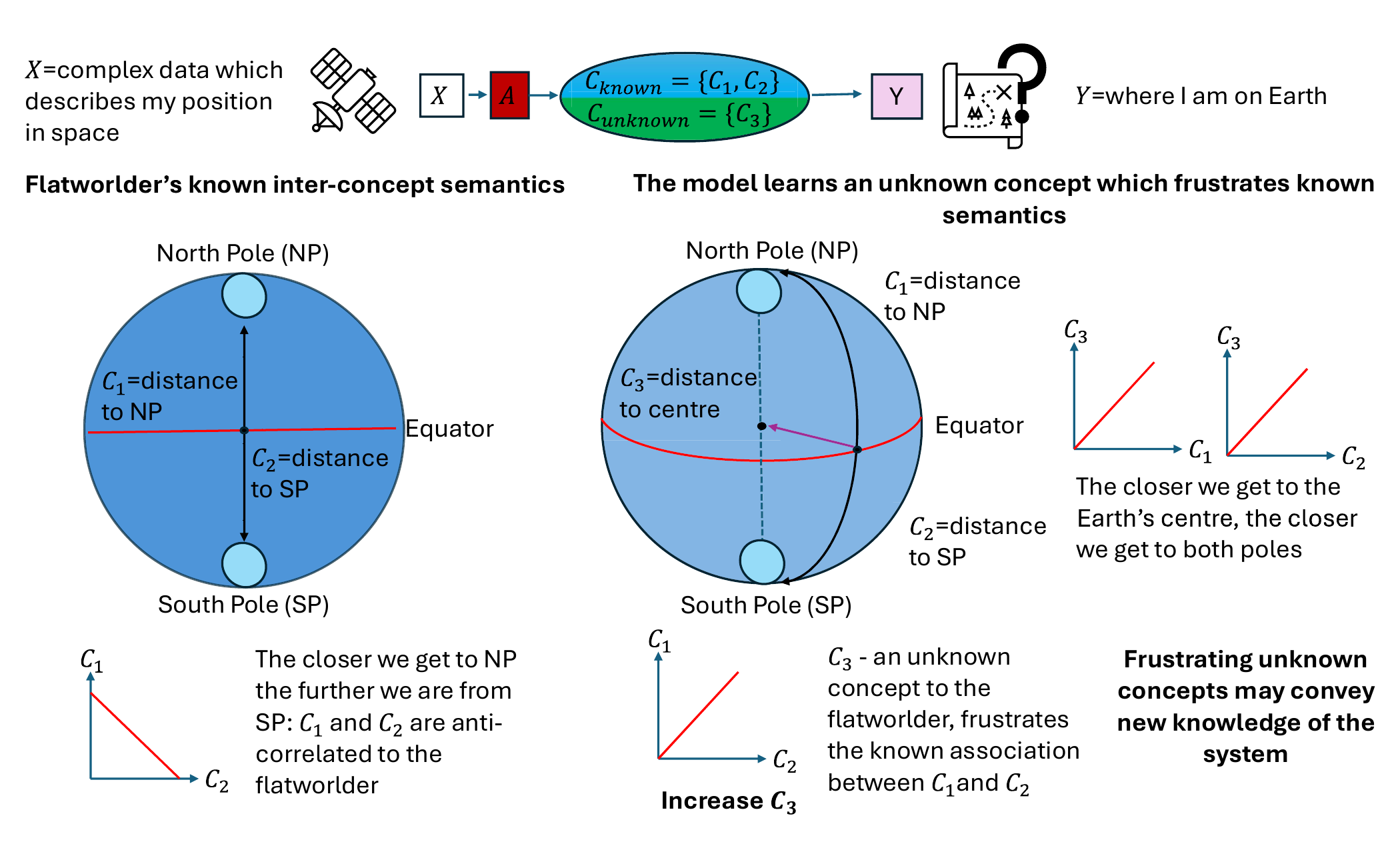}
    \vspace{-0.5em}
    \caption{\textbf{The importance of concept frustration.} Consider a flatworlder training a model to predict a position on Earth from a complex signal (such as from a satellite), in an interpretable manner. They suggest two known concepts as $C_1$=distance form the North Pole and $C_2$=distance from the South pole. They expect these to be anti-correlated. Their trained model is highly accurate of position, but reasons using an unknown concept $C_3$, which frustrates their known inter-concept semantics. Increasing $C_3$ forces $C_1$ and $C_2$ to be positively correlated. Once $C_3$ is understood as distance to the core of a curved Earth, the frustration makes sense and known understanding of the system is improved.}
    \label{fig:frustration}
    \vspace{-0.5em}
\end{figure}
%TC:endignore
\subsection{Background and models}

Let $X$ denote a space of input data. An arbitrary feature extractor (e.g., a foundation model) maps inputs to an activation space $f: X \to A \subset \mathbb{R}^r$. For $\mbf{x} \in X$, we write $\mbf{a} = f(\mbf{x}) \in A$ for its activation vector (Figure \ref{fig:overview_models}A). We treat $A$ as a dense representation of features relevant for predicting downstream tasks.

Let $Y$ denote an outcome variable and $C = \bigcup_{j=1}^k C_j$ a set of $k$ concepts. We assume that $Y$ and $C$ are associated, such that $p(Y|C) \not=p(Y)$. Similarly we assume that $A$ and $C$ are associated, i.e., $p(C|A) \not=p(C)$, which naturally implies an association between $A$ and $Y$ such that $p(Y|A)\not =p(Y)$.

We partition $C$ into $k_{\text{known}}$ known concepts,
\[
C_{\text{known}} = \bigcup_{j=1}^{k_{\text{known}}} C_j,
\]
and $k - k_{\text{known}}$ unknown concepts,
\[
C_{\text{unknown}} = \bigcup_{j=k_{\text{known}}+1}^{k} C_j.
\]
%so that $C = C_{\text{known}} \cup C_{\text{unknown}}$. 
For each $\mbf{a} \in A$, we assume $C_{\text{known}}$ is observed, while $C_{\text{unknown}}$ is not.

A common goal of interpretable modelling is to predict $Y$ from $A$ via intermediate concepts $C$, enabling human-understandable reasoning (Figure \ref{fig:overview_models}B). We consider three modelling strategies. Further architectural and training details are provided in the Methods.

\begin{enumerate}

\item \textbf{Black-box model.}  
We directly approximate $p(Y \mid A)$ using a model trained on $(\mbf{a}_i, y_i)_{i=1}^{N_{\text{train}}}\subset (A,Y)$. In our implementation, we use a single hidden layer of width $h$ with ReLU activation and consider $Y$ as a binary task (extension to non-binary is straightforward). The hidden layer outputs logits $l(\mbf{a})$, from which the predicted task label is obtained via a sigmoid transformation, $\hat{y} = \sigma(l(\mbf{a}))$.
This approach can achieve high predictive accuracy, but offers no direct interpretability (Figure \ref{fig:overview_models}C). See \ref{sec: model architecture} for further details.

\item \textbf{Sparse autoencoder (SAE).}  
We extract $k_{\text{SAE}}$ unsupervised latent features from $A$, using a model trained on $(\mbf{a}_i)_{i=1}^{N_{\text{train}}}\subset A$. An encoder $s: \mathbb{R}^{r} \rightarrow \mathbb{R}^{k_{\text{SAE}}}$ maps activations to the latent feature space
\[
s(\mbf{a}) = \mathrm{ReLU}(\mbf{W}_{\text{SAE}} \mbf{a} + \mbf{b}_{\text{SAE}}) \in \mathbb{R}^{k_{\text{SAE}}},
\]
where $\mbf{W}_{\text{SAE}} \in \mathbb{R}^{k_{\text{SAE}}\times r}$ is a learned matrix of weights and $\mbf{b}_{\text{SAE}} \in \mathbb{R}^{k_{\text{SAE}}}$ is a learned vector of biases, and a linear decoder reconstructs
\[
\hat{\mbf{a}} = s(\mbf{a}) \mbf{D},
\]
where $\mbf{D} \in \mathbb{R}^{k_{\text{SAE}} \times r}$ is a dictionary whose rows define atoms in activation space.

The SAE decomposes $A$ into unsupervised features and may loosely estimate $p(C|A)$, however these features are not conditioned on relevance to the task $Y$ (Figure \ref{fig:overview_models}D).

\item \textbf{Concept bottleneck model (CBM).} Known concepts are first predicted from activations, and these estimated concepts are subsequently used to predict the outcome. We consider here joint training from $(\mbf{a}_i,\mbf{c}_i^{known}, y_i)_{i=1}^{N_{\text{train}}}\subset (A,C_{known},Y)$, and assume $A$ is a sufficiently processed representation that concepts are linearly deducible. The CBM thus learns a linear concept mapping
\[
\hat{\mbf{c}}_{\text{known}}(\mbf{a}) = \mbf{Q} \mbf{a},
\]
with $\mbf{Q} \in \mathbb{R}^{k_{\text{known}} \times r}$, together with a linear task head that maps the predicted concepts to a binary outcome via a sigmoid transformation, yielding $\hat{y}(\mbf{a})$.

The CBM estimates $p(C_{\text{known}} \mid A)$ and $p(Y \mid \hat{C}_{\text{known}})$. When the task depends strongly on $C_{\text{unknown}}$, predictive accuracy may degrade relative to the black-box model. However, the CBM provides direct interpretability by expressing predictions as linear functions of supervised concepts (Figure \ref{fig:overview_models}E).

\end{enumerate}
%TC:ignore
\begin{figure}[H]
    \centering
    \includegraphics[width=\linewidth]{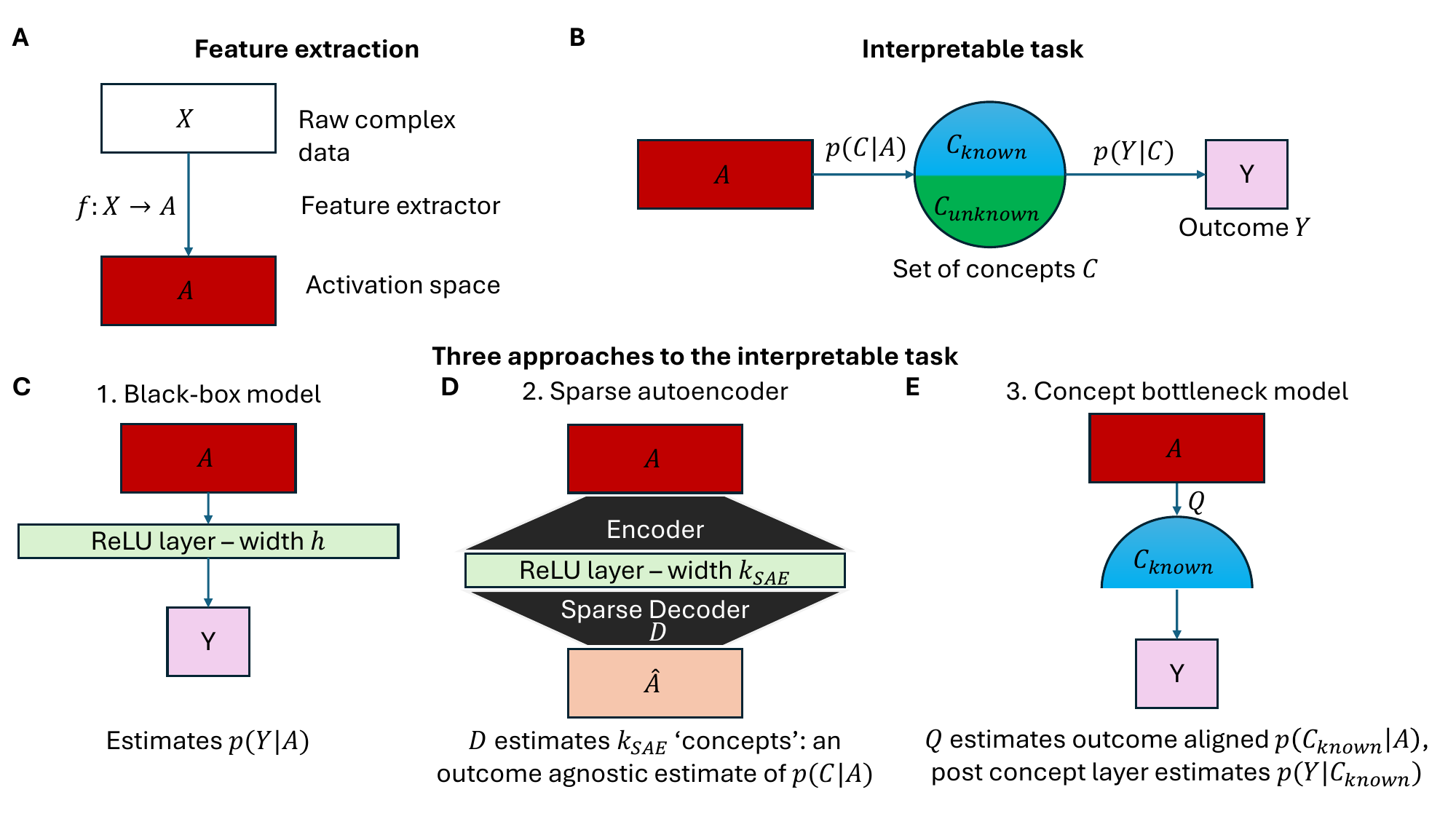}
    \vspace{-0.5em}
    \caption{\textbf{Overview of approach.} (A) Raw data is typically processed by a feature extractor (such as a foundation model) to produce an activation vector. (B) Activations can be used to predict an outcome reasoned on a set of intermediate concepts some which may be known prior to predicting the task. This task can be approached using (C) a black box model, (D) a sparse autoencoder (E) a concept bottleneck model. }
    \label{fig:overview_models}
    \vspace{-0.5em}
\end{figure}
%TC:endignore
\subsection{A task-aligned geometry to compare supervised and unsupervised concepts}

The three models provide three outputs:
\begin{enumerate}
    \item Black box model: An estimate of $p(Y|A)$.
    \item SAE: A matrix $\mbf{D} \in\mathbb{R}^{k_{\text{SAE}}\times r}$, describing a dictionary of $k_{\text{SAE}}$ \emph{unsupervised} concepts extracted from $A$, represented as $r$-dimensional vectors.
    \item CBM: A matrix $\mbf{Q}  \in \mathbb{R}^{k_{known} \times r}$ describing a dictionary of $k_{known}$ \emph{supervised} concepts extracted from $A$, represented as $r$-dimensional vectors.
\end{enumerate}

We wish to compare the known concepts in $\mbf{Q}$ to the unsupervised concepts in $\mbf{D}$. 

The simplest way is using Euclidean geometry over $\mathbb{R}^r$, defining a standard inner product as:
\begin{align}
    \langle \mbf{q}_i, \mbf{d}_j\rangle_E= \mbf{q}_i^\top \mbf{d}_j
\end{align}
where $\mbf{q}_i$, $\mbf{d}_j$ are the $i^{th}$ and $j^{th}$ rows of $\mbf{Q}$ and $\mbf{D}$ respectively, this measures alignment of the SAE and CBM derived concepts in Euclidean space.

However, this geometry is not anchored to the predictive task. Many SAE features may be irrelevant to $Y$, and Euclidean alignment need not reflect task-relevant relationships. Consequently, Euclidean similarity does not guarantee semantically meaningful concept comparisons with respect to the outcome.

A more natural geometry is that of the statistical manifold $(A,p(y|\mbf{a}))$, equipped with the Fisher information metric
\begin{equation}
    F_A(\mbf{a}) = \mathbb{E}_{y \sim p(y|\mbf{a})}\left[\nabla_\mbf{a} \log p(y|\mbf{a})\nabla_\mbf{a} \log p(y|\mbf{a})^\top\right].
\end{equation}

Under our single-layer black-box model, this metric admits a closed-form expression.

\begin{theorem}[Closed-form Fisher metric for one-hidden-layer binary model]
Consider the black-box model defined above with sigmoid output. The Fisher information metric on activation space satisfies
\[
F_A(\mbf{a}) = p(\mbf{a})(1-p(\mbf{a}))\, g(\mbf{a}) g(\mbf{a})^\top,
\]
where $p(\mbf{a})\equiv p(y=1|\mbf{a}) = \sigma(l(\mbf{a}))$ and $g(\mbf{a}) = \nabla_\mbf{a} l(\mbf{a})$.
\end{theorem}

The proof and explicit form of $g(\mbf{a})$ are provided in the Methods.

For fixed $\mbf{a}$, $F_A(\mbf{a})$ is a rank-one positive semi-definite matrix, as it is the outer product of $g(\mbf{a})$ with itself. Thus the pointwise Fisher metric captures sensitivity only along a single direction in activation space. Moreover, 
whenever $p(\mbf{a}) \in \{0,1\}$ the prefactor vanishes and $F_A(\mbf{a}) = 0$.  
Consequently, the Fisher geometry is expected to be more informative near the decision boundary.

We therefore define a global, task-aligned quadratic form by averaging over activations near the decision boundary. Let
\[
A_{\mathrm{ave}} := \{ \mbf{a} \in A_{\mathrm{train}} : p(\mbf{a}) \in [p_{\mathrm{low}}, p_{\mathrm{high}}] \}.
\]
We define
\begin{equation} \label{ave_Fisher}
    \bar{F}_A
    = \frac{1}{|A_{\mathrm{ave}}|}
      \sum_{\mbf{a} \in A_{\mathrm{ave}}}
      p(\mbf{a})(1-p(\mbf{a}))\, g(\mbf{a}) g(\mbf{a})^\top.
\end{equation}

%We emphasise that $\bar{F}_A$ is not intended as a Riemannian metric on the full statistical manifold, but as a task-aligned quadratic form on activation space. Our goal is not to characterise local curvature at a single activation, but to obtain a stable global geometry for comparing concept directions, weighted by their contribution to predictive uncertainty.

Equipped with \eqref{ave_Fisher}, we define the task-aligned inner product
\begin{align}
    \langle \mbf{q}_i, \mbf{d}_j\rangle_{F_A} = \mbf{q}_i^\top \bar{F}_A \mbf{d}_j.
\end{align}
This inner product will detect alignment between SAE and CBM derived concepts in a geometry aligned to task prediction and can thus be expected uncover more task-relevant inter-concept semantics than a Euclidean geometry.

\subsection{Defining and measuring concept frustration}

We now formalise a framework to compare supervised and unsupervised concept directions under $\bar{F}_A$.

We define two concept similarity matrices. First, the task aligned similarity between known and unsupervised concepts $\mbf{S}^{F_A}$ with elements:
\begin{equation}
    S^{F_A}_{ij}
    =
    \frac{\mbf{q}_i^\top \bar{F}_A \mbf{d}_j}
    {\|\mbf{q}_i\|_{\bar{F}_A}\,\|\mbf{d}_j\|_{\bar{F}_A}},
\end{equation}
where $\|\mbf{v}\|_{\bar{F}_A} := \sqrt{\mbf{v}^\top \bar{F}_A \mbf{v}}$ and  
$S^{F_A}_{ij} \in [-1,1]$ measures alignment between $\mbf{q}_i$ and $\mbf{d}_j$ in the task-aligned geometry.

Second, the task-aligned similarities between supervised concepts $\mbf{Z}^{F_A}$ with elements:
\begin{equation}
    Z^{F_A}_{rl}
    =
    \frac{\mbf{q}_r^\top \bar{F}_A \mbf{q}_l}
    {\|\mbf{q}_r\|_{\bar{F}_A}\,\|\mbf{q}_l\|_{\bar{F}_A}},
\end{equation}
$Z^{F_A}_{rl} \in [-1,1]$ quantifies the task-aligned relation between known concepts $\mbf{q}_r$ and $\mbf{q}_l$.

For comparison, we define task-agnostic Euclidean similarities:
\begin{equation}
    S^{E}_{ij}
    =
    \frac{\mbf{q}_i^\top \mbf{d}_j}
    {\|\mbf{q}_i\|_2\,\|\mbf{d}_j\|_2},
\end{equation}
with $\mbf{Z}^{E}$ defined analogously for supervised concepts.  

We now formally define concept frustration and related quantities:

\begin{definition}[Concept triplet frustration]
Let $r \neq l$ index two known concepts and let $j$ index an unsupervised direction.  
The triplet $(\mbf{q}_r,\mbf{q}_l,\mbf{d}_j)$ exhibits concept frustration if
\begin{equation}
    \label{eq_concept_triplet_frustration}
    \mathrm{sign}\big(Z^{F_A}_{rl}\big)
    \neq
    \mathrm{sign}\big(S^{F_A}_{rj} S^{F_A}_{lj}\big).
\end{equation}
\end{definition}

%That is, the task-aligned relation between the known concepts contradicts the joint alignment they exhibit with a common unsupervised direction. 
This definition mirrors frustration in spin systems \cite{SpinGlassTheoryAndBeyond}, where spins are discrete random variables with covariances determined by pairwise couplings $J_{ij}$. A triangle of sites $\{r,l,j\}$ is frustrated when the three pairwise couplings $J_{rl}, J_{lj}, J_{jr}$ between those sites cannot be satisfied simultaneously. This occurs when
\[
\mathrm{sign}(J_{rl}J_{lj}J_{jr}) = -1,
\]
which is equivalent to equation \eqref{eq_concept_triplet_frustration}.

\begin{definition}[Maximally frustrating direction]

For each known concept pair $(r,l)$, define the maximally frustrating unsupervised direction $j^*(r,l)$ via
\begin{equation}
    j^*(r,l)
    =
    \arg\max_{j :
    \mathrm{sign}(Z^{F_A}_{rl})
    \neq
    \mathrm{sign}(S^{F_A}_{rj}S^{F_A}_{lj})}
    \big|
    S^{F_A}_{rj} S^{F_A}_{lj}
    \big|.
\end{equation}
If no such $j$ exists, the known concept pair is unfrustrated.
\end{definition}

\begin{definition}[Pairwise frustration]

We define the pairwise frustration for a known concept pair $(r,l)$ as the magnitude of the maximally frustrating direction:
\begin{equation}
    \mathrm{Frust}^{F_A}(r,l)
    =
    \begin{cases}
        \big|S^{F_A}_{rj^*} S^{F_A}_{lj^*}\big| & \text{if } j^*(r,l) \text{ exists},\\
        0 & \text{otherwise}.
    \end{cases}
\end{equation}
\end{definition}
\begin{definition}[Global frustration]
We define the global frustration as the average pairwise frustration across all known concept pairs:
\begin{equation}
    \gamma^{F_A}
    =
    \frac{2}{k_{\mathrm{known}}(k_{\mathrm{known}}-1)}
    \sum_{l<r}
    \mathrm{Frust}^{F_A}(r,l).
\end{equation}
Thus $\gamma^{F_A} \in [0,1]$ measures the average magnitude of task-aligned geometric contradiction between supervised concepts mediated by unsupervised directions.
\end{definition}

We note that via Euclidean similarities $\mbf{S}^E$ and $\mbf{Z}^E$, Definitions 1-4 can be restated in Euclidean geometry yielding the Euclidean global frustration $\gamma^E$, for comparison.

% To explicitly compare known and unsupervised concept alignment in task-aligned and Euclidean geometry we define the following measure.
% \begin{definition}[Geometry task-relevance]
%     Geometry task-relevance $G$ quantifies how strongly task-aligned geometry differs from Euclidean geometry via a normalised Frobenius distance:
% \begin{equation}
%     G
%     =
%     \frac{\|\mbf{S}^{F_A} - \mbf{S}^{E}\|_{\mathrm{Frob}}}
%     {\|\mbf{S}^{F_A}\|_{\mathrm{Frob}}}.
% \end{equation}
% Large $G$ indicates substantial structural differences between Euclidean and task-aligned concept interactions.
% \end{definition}

\subsection{Detecting concept frustration in the flatworld}

As a motivating example, we return to our flatworlder thought experiment. Consider three concepts $C_1,C_2,C_3$ predictive of an outcome $Y$, where only $(C_1,C_2)$ are supervised. Let $C_1$ and $C_2$ denote surface-parallel distance to the North and South poles respectively. For a traveler constrained to the Earth's surface, moving closer to the North pole necessarily moves them farther from the South pole, so the two distances behave as opposing concepts, i.e.,
\[
\mathrm{Cov}(C_1,C_2) < 0.
\]

We now allow the traveler to descend below the surface. In a spherical (round-world) geometry of radius 1, depth changes the surface geometry. Defining depth of a point as its inward distance from the surface, a traveler at depth $d$ lies on a sphere of radius $1-d$, and the supervised concepts are defined as geodesic distances on that sphere to the North and South poles (Figure \ref{fig:treasure}A). Increasing depth therefore shortens both geodesic distances simultaneously. We denote by $C_3$ depth below the surface and assume this concept is unsupervised. On a round Earth this induces
\[
\mathrm{Cov}(C_1,C_3) < 0,
\qquad
\mathrm{Cov}(C_2,C_3) < 0.
\]
Together with $\mathrm{Cov}(C_1,C_2) < 0$, this sign pattern corresponds to the triplet-frustration structure defined in the previous section.

In contrast, under a cylindrical (flatworld) geometry, depth corresponds to translation orthogonal to a flat disk. The supervised concepts are defined as planar distances within the disk at fixed depth. In this case, depth does not alter surface geometry and is independent of the planar distances, so
\[
\mathrm{Cov}(C_1,C_3) = 0,
\qquad
\mathrm{Cov}(C_2,C_3) = 0,
\]
and no covariance-based concept frustration is expected.

This motivates a simple ``treasure hunter'' task in which we determine whether a treasure is within a radius of our location from a signal encoding its position, using an interpretable model with known concepts $(C_1,C_2)$.

Formally, an activation $\mbf{a}$ encodes a 3D location $\mbf{p}$ and the goal is to predict whether $\mbf{p}$ lies within distance $R$ of a fixed reference point $\mbf{e}=(1,0,0)$:
\[
Y(\mbf{p}) = \mathds{1}\{\|\mbf{p}-\mbf{e}\|_2 < R\}.
\]

We compare two geometries: cylindrical and spherical. In both cases we define the same supervised concepts abstractly as surface-parallel distances to the North and South poles,
\[
C_1(\mbf{p})=\mathrm{dist}_{\mathrm{surface}}(\mbf{p},\mathrm{North}),
\qquad
C_2(\mbf{p})=\mathrm{dist}_{\mathrm{surface}}(\mbf{p},\mathrm{South}),
\]
with depth $C_3$ treated as unknown. We consider a linear-Gaussian mapping from concepts to activations to generate the signal $\mbf{a}$. Details are provided in the Methods.

\begin{figure}[H]
    \centering
    \includegraphics[width=\linewidth]{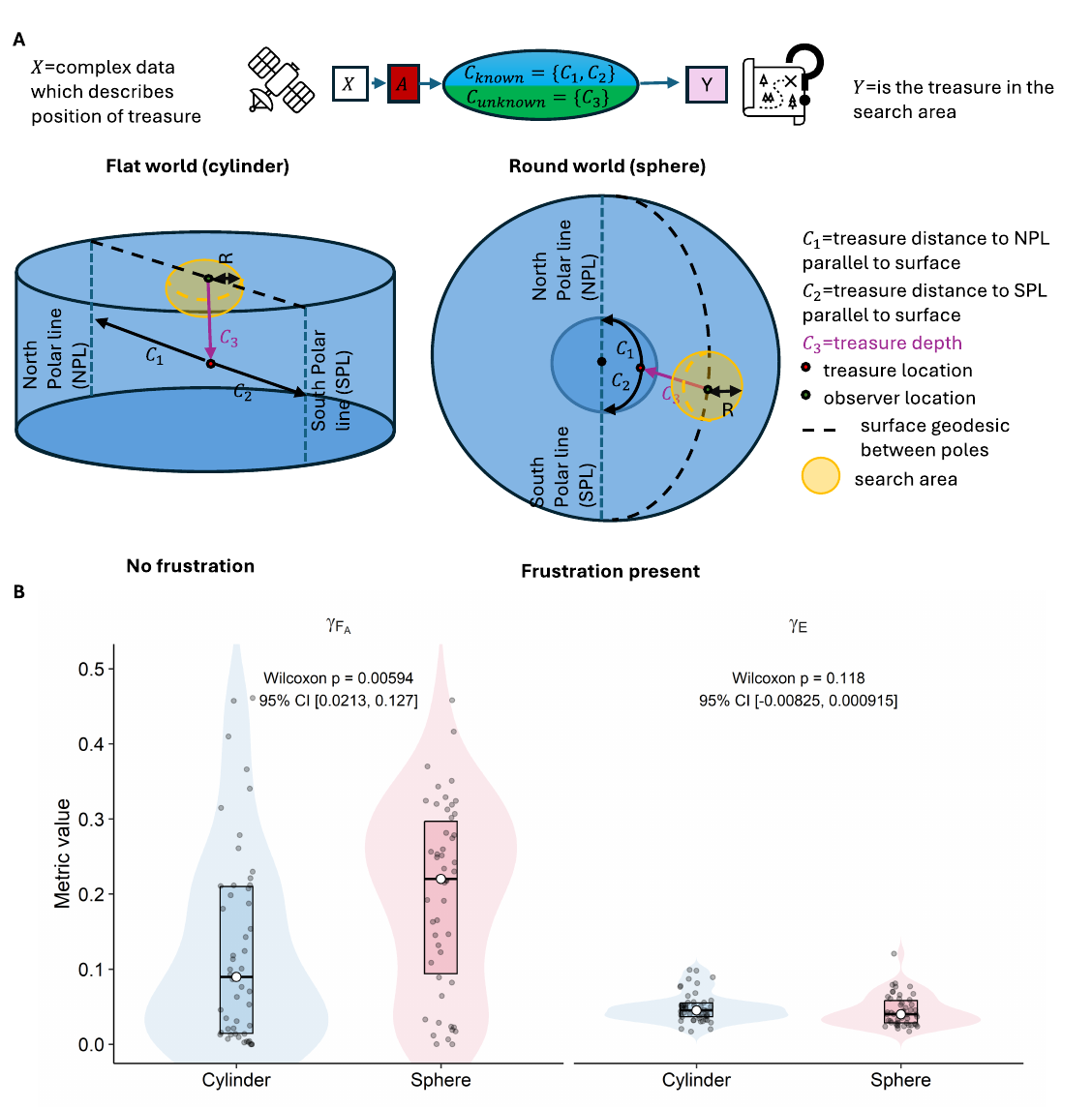}
    \vspace{-0.5em}
    \caption{\textbf{The treasure hunter task and frustration.} A. Consider that we receive a complex signal encoding the position of a treasure, from this signal we wish to predict the binary task of identifying whether a treasure is in a fixed search radius of our position. We employ an interpretable model where known concepts $C_1$ and $C_2$ are the distance to the North and South polar lines parallel to the Earth's surface, and we consider the unknown concept $C_3$ to denote depth below the surface. We consider two geometries: a flat Earth in which there is no frustration between known and unknown concepts, and a round Earth in which there is frustration. B. Box and violin plots display $\gamma_{F_A}$ and $\gamma_E$ calculated from 50 simulations of the flat and round Earth geometry treasure hunter tasks, alongside paired Wilcoxon $p$-values and 95\% confidence intervals of the median difference between round and flat Earth metric values. We see that $\gamma_{F_A}$ detects frustration while $\gamma_{E}$ does not.}
    \label{fig:treasure}
    \vspace{-0.5em}
\end{figure}

Across 50 simulations per geometry, we trained the black-box model, SAE, and CBM and computed the frustration metrics defined above. Figure~\ref{fig:treasure}B shows that the task-aligned frustration score $\gamma^{F_A}$ is significantly higher in the spherical setting than in the cylindrical (paired Wilcoxon $p=5.4\times10^{-3}$), consistent with stronger task-aligned concept frustration in the round-world setting. In contrast, the Euclidean analogue $\gamma^E$ does not distinguish the two settings (paired Wilcoxon $p=0.118$). This indicates that frustration is not detectable from ambient geometric alignment alone, but emerges when concept directions are evaluated in a geometry weighted by their contribution to task uncertainty. Further comparisons are provided in Supplementary Table S1.

\subsection{A data model to control concept frustration}
\label{subsec_data_model}

To study concept frustration under controlled conditions, we construct a synthetic data generator in which the covariance structure of supervised and unsupervised concepts can be manipulated explicitly while preserving the marginal structure of known concepts (see Methods for details). 

We consider a concept vector
\[
\bsym{\chi}=(\bsym{\chi}_k,\bsym{\chi}_u)\in\mathbb{R}^k,
\]
where $\bsym{\chi}_k$ contains $k_{\mathrm{known}}$ supervised concepts and $\bsym{\chi}_u$ contains $k_{\mathrm{unknown}}$ unsupervised concepts. Concepts are sampled from a zero-mean Gaussian distribution with covariance
\[
\mbf{B}(\alpha)=
\begin{pmatrix}
\mbf{B}_{kk} & \mbf{B}_{ku}(\alpha) \\
\mbf{B}_{uk}(\alpha) & \mbf{B}_{uu}(\alpha)
\end{pmatrix},
\]
where $\mbf{B}_{kk}$ is the covariance among known concepts and $\mbf{B}_{ku}(\alpha)$ the cross-covariance between known and unknown concepts. The parameter $\alpha \in [-1,1]$ enforces three regimes. When $\alpha=0$, known and unknown concepts are independent. For $\alpha<0$, cross-block correlations reinforce the correlation structure within $\mbf{B}_{kk}$, producing a non-frustrated regime. For $\alpha>0$, they oppose it, generating concept frustration.

To ensure $\mbf{B}(\alpha)$ is positive semi-definite, we take $\mbf{B}_{uu}(\alpha)$ to be of the form
\[
\mbf{B}_{uu}(\alpha)
=
\mbf{B}_{\mathrm{temp}}
+
\mbf{B}_{uk}(\alpha)\mbf{B}_{kk}^{-1}\mbf{B}_{ku}(\alpha),
\]
with $\mbf{B}_{\mathrm{temp}}$ an $\alpha$-independent baseline covariance capturing variation intrinsic to the unsupervised concepts. %The scalar parameter $\alpha$ therefore changes only the coupling between known and unknown concepts, while leaving the marginal covariance of known concepts $\mbf{B}_{kk}$ fixed across regimes.

We define the task via
\[
\tau=\bsym{\psi}_k^\top \bsym{\chi}_k+\bsym{\psi}_u^\top \bsym{\chi}_u+\eta,
\qquad
\eta\sim\mathcal{N}(0,\sigma_y^2),
\]
where $\bsym{\psi}=(\bsym{\psi}_k,\bsym{\psi}_u)$ specifies the task-relevant directions on known and unknown concepts, and $\sigma_y^2$ controls label noise. Binary labels are then generated as
$
y=\mathds{1}\{\tau>0\}.
$
A separate parameter $\omega$ controls how strongly the task depends on known versus unknown concepts through the construction of $\bsym{\psi}$. Activations are generated from concepts through a linear-Gaussian map.

This provides a controlled environment in which (i) the supervised concept ontology remains fixed across regimes, (ii) frustration strength is governed solely by $\alpha$, and (iii) the task dependence on unknown structure is governed by $\omega$. 

% This synthetic model separates three effects that are often confounded in real data: the covariance structure of supervised concepts ($\mbf{B}_{kk}$), the recoverable influence of unknown concepts through cross-covariances ($\mbf{B}_{ku}(\alpha)$), and the task dependence on hidden concept directions ($\bsym{\psi}_u$). As shown below, these quantities jointly determine the concept-optimal accuracy achievable using supervised concepts alone.

\subsection{Concept-optimal  accuracy can be interpreted in terms of frustration}

We now consider the concept-optimal Bayes classifier for the synthetic data defined in Section \ref{subsec_data_model}, that minimises the expected error of predicting the task, given knowledge of supervised concepts. We find that the accuracy of this classifier admits a closed form expression.

\begin{theorem}[Concept-optimal Accuracy in the Linear–Gaussian Model]
\label{thm:optimal_cbm_accuracy}
Consider the synthetic model defined in Section \ref{subsec_data_model}.
% Let the concept vector be denoted $\bsym{\chi}=(\bsym{\chi}_k,\bsym{\chi}_u)$ where $\bsym{\chi}_{k}=(c_1,\cdots,c_{k_{\mathrm{known}}})$ denotes known concepts and $\bsym{\chi}_{u}=(c_{k_{\mathrm{known}}+1},\cdots,c_{k})$ unknown concepts, under our model $\bsym{\chi} = (\bsym{\chi}_k,\bsym{\chi}_u) \sim \mathcal{N}(0,\mbf{B})$ with block covariance
% \[
% \mbf{B} =
% \begin{pmatrix}
% \mbf{B}_{kk} & \mbf{B}_{ku} \\
% \mbf{B}_{uk} & \mbf{B}_{uu}
% \end{pmatrix},
% \]

% where
% \[
% \mbf{B}_{uu}(\alpha)
% =
% \mbf{B}_{temp} + \mbf{B}_{uk}(\alpha) \mbf{B}_{kk}^{-1} \mbf{B}_{ku}(\alpha),
% \]
% and $\mbf{B}_{kk}$ is positive definite (and thus invertible).

% Let the weight vector be denoted $\bsym{\psi} =(\bsym{\psi}_k,\bsym{\psi}_u)$ where $\bsym{\psi}_{k}=(w_1,\cdots,w_{k_{\mathrm{known}}})$ denotes known concept-task weights and $\bsym{\psi}_{u}=(w_{k_{\mathrm{known}}+1},\cdots,w_{k})$ unknown concept-task weights.

% Let the task variable be
% \[
% \tau = \bsym{\psi}_k^\top \bsym{\chi}_k + \bsym{\psi}_u^\top \bsym{\chi}_u + \eta,
% \qquad
% \eta \sim \mathcal{N}(0,\sigma_y^2),
% \]
% and define the binary label
% \[
% y = \mathds{1}\{\tau > 0\}.
% \]
Then the  concept-optimal Bayes classifier using only the known concepts $\bsym{\chi}_k$,
which minimises the risk
\[ R( \hat y) = \mathbb{E}_{(\bsym{\chi}_k,y)} \left[\mathds{1}\{ \hat{y}(\bsym{\chi}_k) \neq y \} \right],\]
 is
\[
\hat{y}(\bsym{\chi}_k)
=
\mathds{1}\{\bsym{\phi}(\alpha)^\top \bsym{\chi}_k > 0\},
\]
and its classification accuracy is
\[
\mathrm{Acc}_{\textrm{CBM}}(\alpha)
=
\frac{1}{2}
+
\frac{1}{\pi}
\arctan
\left(
\sqrt{
\frac{
\bsym{\phi}(\alpha)^\top \mbf{B}_{kk} \bsym{\phi}(\alpha)
}{
\bsym{\psi}_u^\top \mbf{B}_{temp} \bsym{\psi}_u + \sigma_y^2
}
}
\right),
\]
where
\[
\bsym{\phi}(\alpha)
=
\bsym{\psi}_k + \mbf{B}_{kk}^{-1} \mbf{B}_{ku}(\alpha) \bsym{\psi}_u,
\]

\end{theorem}

The proof of this theorem is provided in the Methods.

The denominator term $\bsym{\psi}_u^\top \mbf{B}_{temp} \bsym{\psi}_u + \sigma_y^2$ is independent of $\alpha$, thus
\[
\mathrm{Acc}_{\textrm{CBM}}(\alpha_1)
>
\mathrm{Acc}_{\textrm{CBM}}(\alpha_2)
\quad
\Longleftrightarrow
\quad
\bsym{\phi}(\alpha_1)^\top \mbf{B}_{kk} \bsym{\phi}(\alpha_1)
>
\bsym{\phi}(\alpha_2)^\top \mbf{B}_{kk} \bsym{\phi}(\alpha_2).
\]

Expanding this term yields
\begin{align*}
  \bsym{\phi}(\alpha)^\top \mbf{B}_{kk}\bsym{\phi}(\alpha)
  &= \langle \bsym{\psi}_k,\mbf{B}_{kk}\bsym{\psi}_k\rangle_E + 2\langle \bsym{\psi}_k,\mbf{B}_{ku}(\alpha)\bsym{\psi}_u\rangle_E +\langle\bsym{\psi}_u,\mbf{B}_{uk}(\alpha)\mbf{B}_{kk}^{-1}\mbf{B}_{ku}(\alpha) \bsym{\psi}_u\rangle_E.\\
  &= T_1+2T_2+T_3
\end{align*}
This admits a geometric interpretation of optimal accuracy.
\begin{enumerate}
    \item Supervised signal:
    \[
    T_1:=\langle \bsym{\psi}_k, \mbf{B}_{kk} \bsym{\psi}_k \rangle_E = ||\bsym{\psi}_k||^2_{\mbf{B}_{kk}}.
    \]
    This measures task-relevant signal present in known concepts alone. Geometrically, it is the squared length of the supervised task direction $\bsym{\psi}_k$ in the covariance geometry induced by $\mbf{B}_{kk}$. 
    This term is invariant under frustration.
    \item Cross-alignment between supervised and unknown signal:
    \[
    T_2(\alpha) := \langle \bsym{\psi}_k, \mbf{B}_{ku}(\alpha)\bsym{\psi}_u \rangle_E.
    \]
    This term captures how the unknown task direction $\bsym{\psi}_u$ projects into known-concept space through the cross-covariance matrix $\mbf{B}_{ku}(\alpha)$, and whether this projection reinforces or interferes with the supervised task direction $\bsym{\psi}_k$. 
    Positive alignment increases accuracy, whereas negative alignment reduces it. 
    Frustration acts primarily by modifying this term.
    \item Recoverable unknown signal:
    \[
    T_3(\alpha):= \langle \bsym{\psi}_u, \mbf{B}_{uk}(\alpha)\mbf{B}_{kk}^{-1}\mbf{B}_{ku}(\alpha) \bsym{\psi}_u \rangle_E=||\bsym{\psi}_u||^2_{\mbf{B}_{uk}(\alpha)\mbf{B}_{kk}^{-1}\mbf{B}_{ku}(\alpha)}.
    \]
    %This matrix is positive semi-definite and hence $T_3(\alpha) \ge 0$. 
    This term quantifies the magnitude of unknown task-relevant variation that becomes linearly recoverable from known concepts via cross-block correlations. 
    It measures signal strength rather than directional alignment, and is zero if known and unknown concepts are independent ($\alpha=0$).
\end{enumerate}

% Thus, frustration influences CBM accuracy through how the unknown task direction is rotated within the covariance geometry of the known concepts. 

The term 
\begin{align*}
   \bsym{\psi}_u^\top \mbf{B}_{temp} \bsym{\psi}_u + \sigma_y^2&= ||\bsym{\psi}_u||_{\mbf{B}_{temp}}^2 + \sigma_y^2\\
    &= T_4 +\sigma_y^2
\end{align*}
also admits a geometric interpretation: $T_4$ is the unknown signal not recoverable from supervised concepts, and the larger this term is the worse concept-optimal accuracy.

Informally, $T_1$ describes known-knowns, $T_2$ and $T_3$ known-unknowns and $T_4$ unknown-unknowns.

% When trained on the synthetic data defined in Section \ref{subsec_data_model}, a generic CBM learns concepts $\hat{\mbf{c}}_{\mathrm{known}}$ with covariance $\mathrm{Cov}(\hat{\mbf{c}}_{\mathrm{known}})$  that may differ from the ground-truth covariance structure $\mbf{B}_{\mathrm{known}}$. To quantify this decrease in interpretability, we define the following score.
% \begin{definition}[Semantic fidelity]
% The semantic fidelity score $\beta$ measures deviation between the covariance of predicted known concepts and the ground-truth covariance structure
% \begin{equation} \label{eq_semantic_fidelity}
% \beta =
% \frac{
% \| \mathrm{Cov}(\hat{\mbf{c}}_{\mathrm{known}}) - \mbf{B}_{\mathrm{known}} \|_{\mathrm{Frob}}
% }{
% \| \mbf{B}_{\mathrm{known}} \|_{\mathrm{Frob}}
% }.
% \end{equation}
% \end{definition}

\subsection{Frustration between known and unknown concepts limits CBM task accuracy}

We trained CBMs, black-box models, and SAEs across 60 parameter regimes of our synthetic data generator (see Methods). Each regime was repeated across 10 independent seeds, yielding 600 matched datasets. Full results are provided in Supplementary Table S2.

%CBM task accuracy decreased as the task became increasingly dependent on unknown concepts (increasing $\omega$). When $\omega=0$, all task-relevant signal was contained within known concepts and CBM accuracy was maximal. As $\omega \to 1$, CBM performance declined (Figure \ref{fig:acc_frust}). % Since known concepts contain no predictive information in this regime, any excess performance must arise from unsupervised task-relevant information encoded within the learned concept representations themselves (concept leakage). Indeed $26.3\%$ of jointly trained CBMs models exceeded Bayes-optimal CBM accuracy given known concepts (mean excess $0.139$), which is only possible under leakage. We also trained CBMs using sequential and independent training strategies (Methods). Leakage was less pronounced under sequential training ($21.2\%$ models exceed Bayes-optimal accuracy, mean excess $0.0541$) and lowest under independent training ( $14.1\%$ models exceed Bayes-optimal accuracy, mean excess $0.0272$) (Figure~\ref{fig:leakage}). This pattern is consistent with expected patterns of leakage across the three training regimes and highlights that CBMs demonstrate leakage to varying degrees, regardless of training strategy. We note that models which showed leakage did not show a consistent association with values of $\alpha$, confirming that leakage and frustration are separate processes (Figure \ref{fig:leakage}).

Frustration significantly reduced CBM task accuracy. CBM accuracy was significantly lower under $\alpha=+1$ compared to both $\alpha=-1$ and $\alpha=0$ (paired Wilcoxon $p<10^{-4}$; Figure~\ref{fig:acc_interpret}A). In contrast, black-box model accuracy did not differ between $\alpha=+1$ and $\alpha=-1$ (paired Wilcoxon $p=0.13$, Figure~\ref{fig:acc_interpret}B), and was slightly worse when $\alpha=0$, indicating that frustration limits CBM performance without increasing intrinsic task difficulty. %In fact frustration may slightly reduce intrinsic task difficulty compared to independent concept regimes.

\begin{figure}[H]
    \centering
    \includegraphics[width=0.94\linewidth]{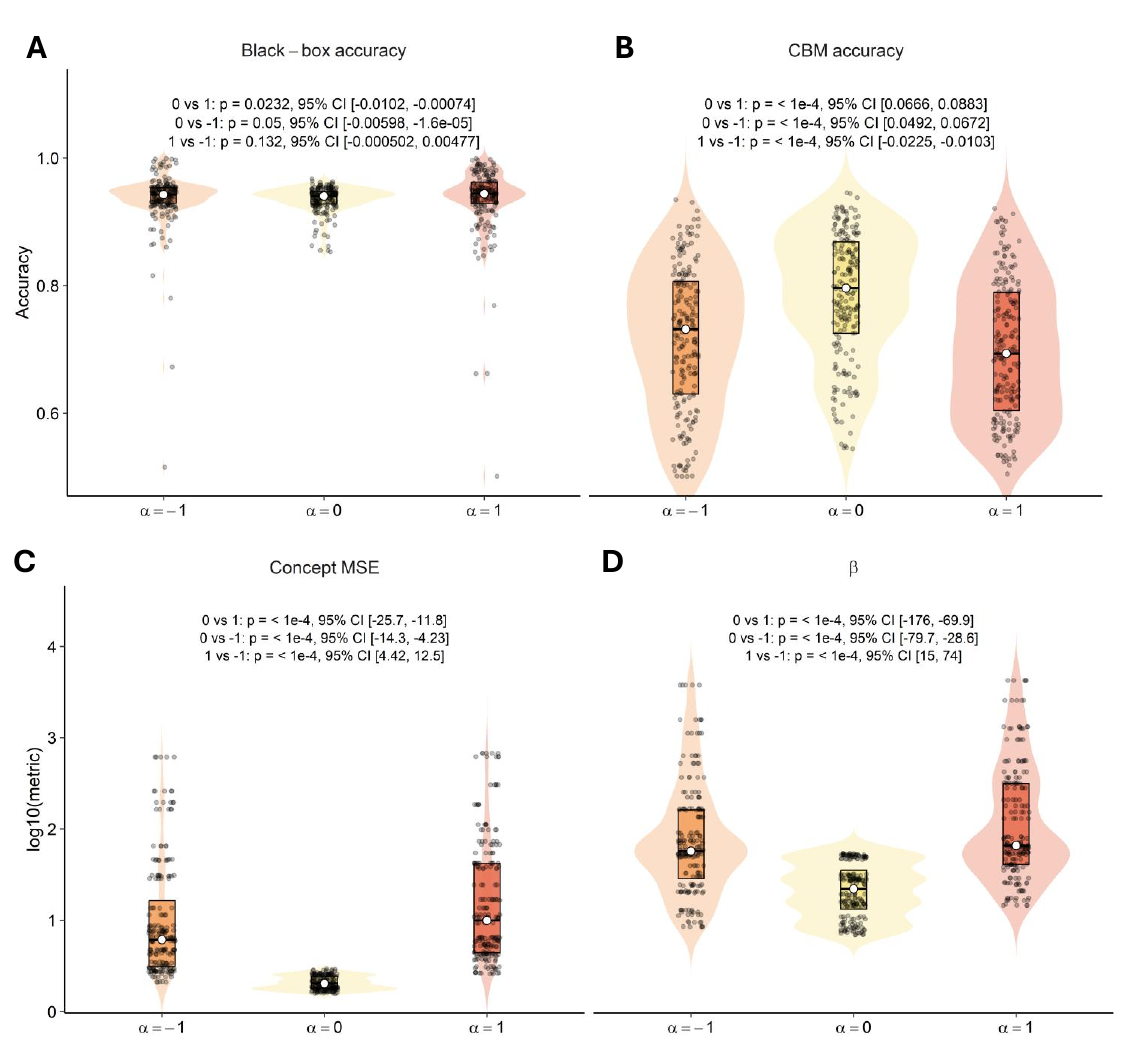}
    \vspace{-0.5em}
    \caption{\textbf{CBM accuracy and interpretability are degraded under frustration.} Box and violin plots display (A) accuracy of black box classifiers (B) accuracy of CBM classifiers (C) concept mean-square error (MSE) for CBM classifiers and (D) concept semantic fidelity metric $\beta$ for $\alpha \in \{-1,0,1\}$, calculated from 600 simulations of our synthetic data generator across a range of parameter regimes, alongside paired Wilcoxon $p$-values and 95\% confidence intervals of the median difference between $\alpha$ values. We see that frustration significantly degrades accuracy and interpretability of CBMs without making the task more difficult for a non-interpretable model}
    \label{fig:acc_interpret}
    \vspace{-0.5em}
\end{figure}

\subsection{Frustration impairs semantic interpretability of CBMs}

We next examined the effect of frustration on concept-level interpretability.

Frustration significantly impaired concept prediction accuracy. Concept mean squared error (MSE) was significantly higher under $\alpha=+1$ compared to both $\alpha=-1$ and $\alpha=0$ (paired Wilcoxon $p<10^{-4}$; Figure \ref{fig:acc_interpret}C).

When trained on the synthetic data, a CBM learns concepts $\hat{\mbf{c}}_{\mathrm{known}}$ with covariance $\mathrm{Cov}(\hat{\mbf{c}}_{\mathrm{known}})$  that may differ from the ground-truth covariance structure $\mbf{B}_{\mathrm{known}}$. To quantify this decrease in interpretability, we define the following score.
\begin{definition}[Semantic fidelity]
The semantic fidelity score $\beta$ measures deviation between the covariance of predicted known concepts and the ground-truth covariance structure
\begin{equation} \label{eq_semantic_fidelity}
\beta =
\frac{
\| \mathrm{Cov}(\hat{\mbf{c}}_{\mathrm{known}}) - \mbf{B}_{\mathrm{known}} \|_{\mathrm{Frob}}
}{
\| \mbf{B}_{\mathrm{known}} \|_{\mathrm{Frob}}
}.
\end{equation}
\end{definition}
Under frustration, $\beta$ increased markedly (paired Wilcoxon $p<10^{-4}$, Figure \ref{fig:acc_interpret}D), indicating distortion of the learned inter-concept relationships.

Thus, frustration degrades both predictive and structural interpretability of CBMs.

\subsection{A task-aligned geometry detects known-unknown concept frustration}

We next evaluated whether frustration can be detected via our metric.

The task-aligned frustration metric $\gamma^{F_A}$ was significantly higher under $\alpha=+1$ compared to both $\alpha=-1$ and $\alpha=0$ (paired Wilcoxon $p<10^{-4}$; Figure \ref{fig:sim_frust}A). In contrast, the Euclidean analogue $\gamma^{E}$ did not reliably distinguish frustration (Figure \ref{fig:sim_frust}A). Moreover, $\gamma^{F_A}$ was elevated when the theoretical cross alignment term $T_2$ was non-zero (paired Wilcoxon $p<10^{-4}$), whereas $\gamma^{E}$ was not (Figure \ref{fig:sim_frust}B). 

Thus, frustration is detectable in task-aligned Fisher geometry but not in Euclidean geometry.

\begin{figure}[H]
    \centering
    \includegraphics[width=\linewidth]{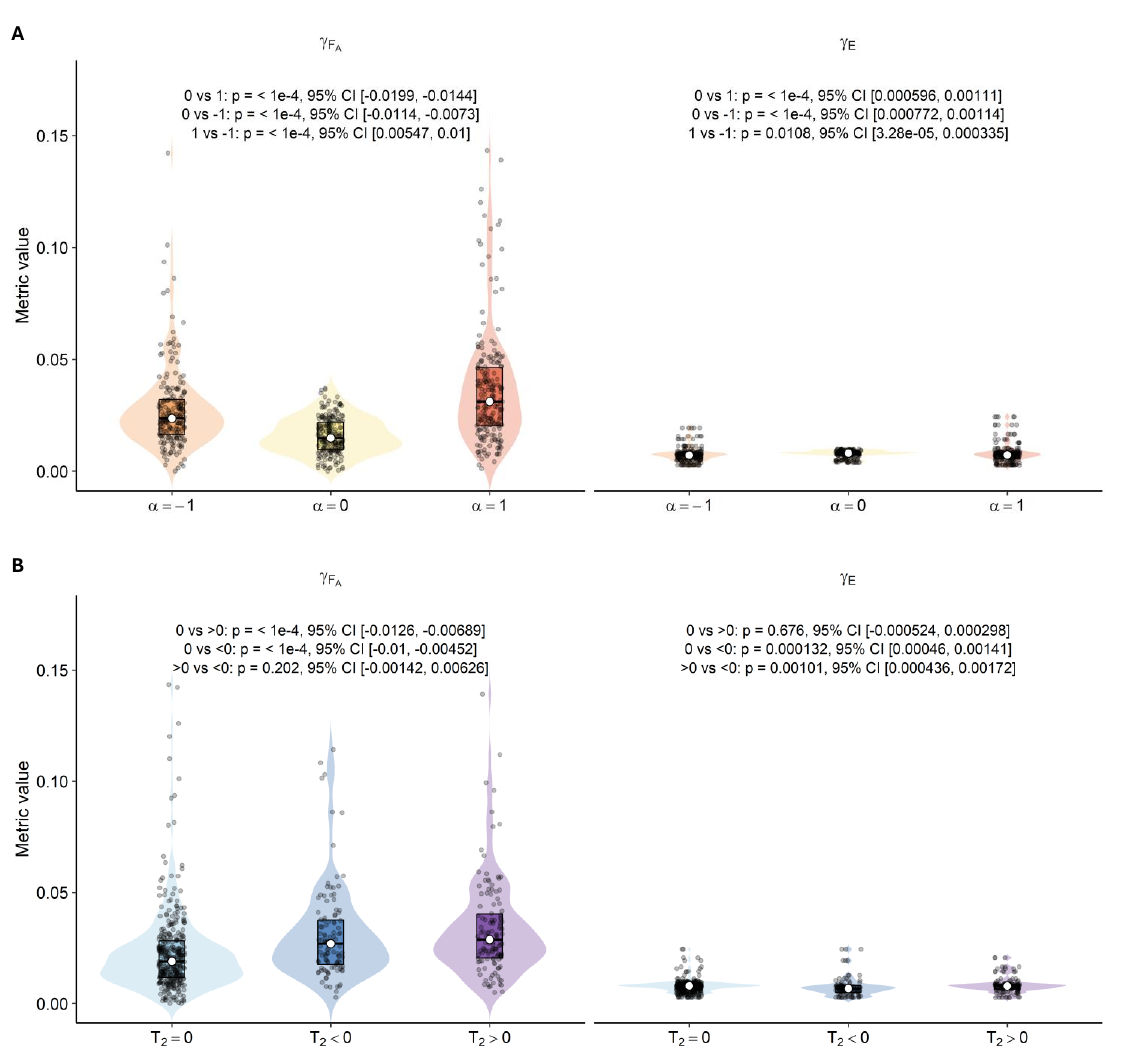}
    \vspace{-0.5em}
    \caption{\textbf{$\gamma_{F_A}$ detects frustration in inter-concept covariance and concept task weights between known and unknown concepts} Box and violin plots display $\gamma_{F_A}$ and $\gamma_E$ calculated from 600 simulations of our synthetic data model across 60 parameter regimes separated by (A) $\alpha$ and (B) the sign of $T_2$ the theoretical cross alignment term between known and unknown signal. Paired Wilcoxon $p$-values and 95\% confidence intervals of the median difference between metric values across various parameter groups are displayed. We see that $\gamma_{F_A}$ is elevated when concept covariance between known and unknown concepts is frustrated, while $\gamma_{E}$ is not. Similarly we see that $\gamma_{F_A}$ is consistently elevated when the cross alignment term $T_2$ is non-zero, while $\gamma_E$ is not.}
    \label{fig:sim_frust}
    \vspace{-0.5em}
\end{figure}

\subsection{Resolving concept frustration in real-world language and vision tasks}

We next evaluated whether concept frustration can be detected in real-world language and vision tasks.% embedded in foundation model representation spaces.

For language we considered sarcasm detection. We used a dataset of 28,000 news headlines labelled with sarcasm \cite{misra2023Sarcasm} and embedded them using the pretrained DeBERTa-v3-base language model \cite{he2021debertav3} (Methods). %, yielding a fixed 768-dimensional representation per headline. 
From a pretrained RoBERTa sentiment classifier \cite{barbieri-etal-2020-tweeteval} we derived three sentiment-related concepts: $C_1$ (positive sentiment), $C_2$ (negative sentiment), and $C_3$ (sentiment intensity). In this dataset, $C_1$ and $C_2$ are inversely correlated, while $C_3$ is positively correlated with both, forming a frustrating triple. A linear model confirmed all three concepts independently associated with sarcasm.

For vision we considered a binary classification task distinguishing gulls from terns in the CUB-200-2011 dataset \cite{wah2011caltech}. Images were embedded using the pretrained CLIP ViT-B/16 vision encoder \cite{Radford2021}. From the CUB attribute annotations we defined three concepts: $C_1$ (white upperparts), $C_2$ (gull-like shape), and $C_3$ (solid wing pattern). Restricting to gull and tern species, $C_1$ and $C_2$ were positively correlated, while $C_3$ was negatively correlated with $C_2$ and positively correlated with $C_1$, forming a frustrating triple. All three concepts were independently associated with the classification outcome.

For each task we trained two CBMs on the same embedding space. CBM1 used $(C_1,C_2)$ as supervised concepts, treating $C_3$ as unknown. CBM2 used all three concepts $(C_1,C_2,C_3)$ as supervised. Both models were compared against the same SAE dictionary $\mbf{D}$ learned from the embedding space. Frustration was computed only on the concept pair $(C_1,C_2)$ for both models as $\mathrm{Frust}^{F_A/E}(1,2)$ .

Across both tasks, $\mathrm{Frust}^{F_A}(1,2)$ was significantly higher in CBM1 than in CBM2 (paired Wilcoxon tests $p=0.0488$, Figures \ref{fig:real_world}A\&B). In contrast, the Euclidean analogue $\mathrm{Frust}^{E}(1,2)$ did not differ between models. Further results are provided in Supplementary Tables S3\&4.%Task accuracy was modestly higher in CBM2 in both tasks, although the difference did not reach statistical significance (Supplementary Tables S3\&4).

These results indicate that when a task-relevant frustrating concept is omitted, the model learns representations of the remaining concepts that are geometrically inconsistent with the fixed unsupervised feature basis. When the frustrating concept is incorporated into the supervised ontology, the learned representation of the shared concepts reorganises and becomes less frustrated with the same SAE directions in the task-aligned geometry. Thus, incorporating a previously hidden concept does not merely improve ontology completeness; it changes how the model internally represents relationships between existing concepts.

This mirrors how human conceptual understanding can change when a previously unrecognised explanatory factor is introduced, reorganising relationships between existing concepts \cite{Gershman2012,Schneider2012}.

\begin{figure}[H]
    \centering
    \includegraphics[width=\linewidth]{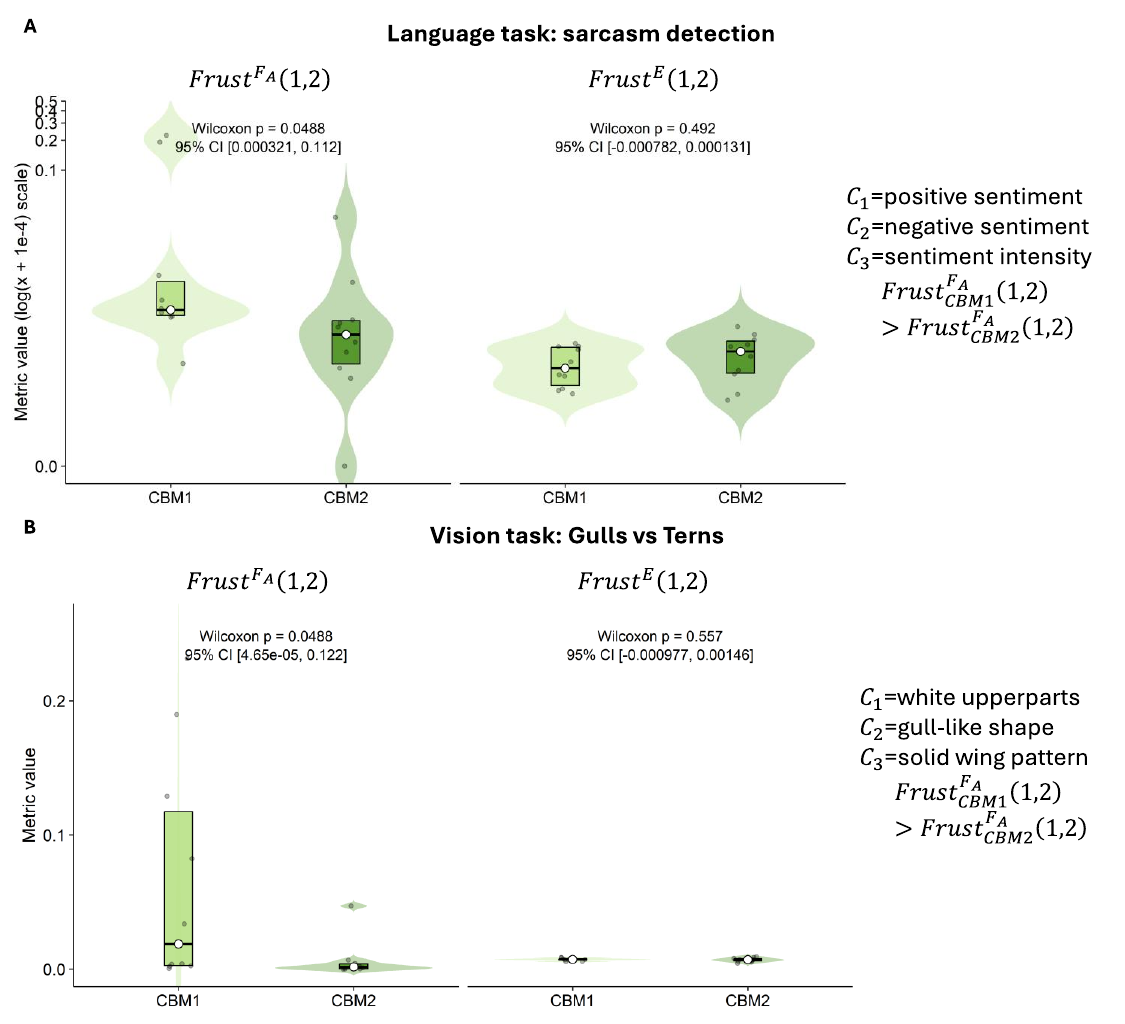}
    \vspace{-0.5em}
    \caption{\textbf{The value of frustration in real world language and vision tasks.} 
    Box and violin plots display $\textrm{Frust}^{F_A}(1,2)$ and $\textrm{Frust}^{E}(1,2)$ for CBM1 and CBM2, calculated from 10 fold cross-validation of models trained to predict (A) a language task: sarcasm detection and (B) a vision task: discriminating pictures of gulls from terns. Paired Wilcoxon $p$-values and 95\% confidence intervals of the median difference between metric values between CBM1 and CBM2 are displayed. We see that $\textrm{Frust}^{F_A}_{CBM1}(1,2)>\textrm{Frust}^{F_A}_{CBM2}(1,2)$ in both settings while $\textrm{Frust}^{E}_{CBM1}(1,2)=\textrm{Frust}^{E}_{CBM2}(1,2)$.}
    \label{fig:real_world}
    \vspace{-0.5em}
\end{figure}

\section{Discussion}

Alignment between human-interpretable reasoning and machine representations remains a central challenge in explainable AI \cite{Banerji2025, Marconato23}. A substantial body of work assumes that unsupervised machine representations are only valuable when orthogonal to human-defined concepts \cite{higgins2017betavae, Chen2020, ismail2024concept, MENG2026108346}. Our results suggest a different perspective: machine-derived concepts can instead induce relationships between known concepts that contradict their expected semantics, motivating a reorganisation of the geometry of the known concept space. We term such concepts \emph{frustrating}. Their presence degrades both accuracy and interpretability of supervised concept models. In real-world vision and language tasks, incorporating a frustrating concept into a supervised ontology reorganises learned representations of known concepts, restoring geometric consistency between human and machine representations.

% Alignment between human-interpretable reasoning and machine representations remains a central challenge in explainable AI \cite{Banerji2025, Marconato23}. A substantial body of work assumes unsupervised machine representations are only valuable when orthogonal to human-defined concepts \cite{higgins2017betavae, Chen2020, ismail2024concept, MENG2026108346}. Our results suggest a different perspective. We highlight that machine concepts may instead induce relationships between known concepts that contradict
% their expected semantics, motivating a reorganisation of the geometry of the known concept space. We say that such machine-concepts are frustrating. We demonstrate that once a frustrating concept is incorporated into the a supervised ontology, the learned representations of the known concepts reorganise and restore geometric consistency between human and machine concept representations.

This process resembles how conceptual change occurs in scientific discovery \cite{kuhn2012structure}, where progress often arises not by introducing independent variables but by discovering a concept that forces reinterpretation of existing knowledge. Recognising that the Earth is round rather than flat resolves
contradictions in geodesic distances observed when moving in depth. %Galileo’s heliocentric modelreorganised the relationships between planetary motions that appeared inconsistent under geocentric models. 
Similarly, Einstein’s theory of relativity reorganised the classical understanding of space, time, and motion, resolving tensions that could not be explained within Newtonian mechanics. Our results suggest interpretable machine learning systems may undergo an analogous process: the discovery or inclusion of a previously hidden concept can reorganise the internal geometry of supervised concept representations and resolve contradictions between human and machine reasoning. We formalise this via the notion of concept frustration.

Though encouraging, our results require some considerations. First, our theoretical analysis relies on a linear–Gaussian generative model which enables analytical tractability. %and yields a closed-form expression for Bayes-optimal CBM accuracy. 
This is not an unrealistic assumption: CBMs typically employ a linear mapping from concepts to outcomes to preserve interpretability \cite{Koh2020}, and foundation models are often assumed to provide embeddings from which relevant concepts are approximately linearly recoverable \cite{choi2024adaptive, bhalla2024interpreting, linearrepresentationhyp}. However, real-world representation spaces may exhibit more complex associations. Second, our measure of pairwise frustration depends on selecting a maximally frustrating unsupervised direction from the SAE dictionary. While this choice captures the strongest geometric contradiction between concept directions, alternative strategies may provide complementary insights. Third, the task-aligned geometry is derived from a quadratic form based on the Fisher information metric of a simple black-box classifier. While this construction provides an analytically tractable geometry, more complex predictive models may induce richer structures.

Promising directions of future work include extending the theory beyond linear–Gaussian settings, analysing higher-order structures in concept similarity matrices, elaborating the connection to frustration in spin systems, and applying frustration-based diagnostics to large foundation models. Our perspective also motivates new model architectures that combine supervised concepts with unsupervised concept discovery. %For instance, concept bottleneck models could be augmented with latent channels learned through sparse autoencoders, not by enforcing independence between known and unknown concepts but by explicitly identifying and resolving frustration between them. 
More broadly, viewing interpretability through the lens of geometric consistency between human concepts and machine representations provides a principled framework for identifying missing conceptual dimensions and refining ontologies through which both humans and machines reason about complex systems.

\section{Methods}

\subsection{Model architectures}
\label{sec: model architecture}

\paragraph{Black box model.} We employ a feedforward architecture with a single hidden layer of width $h$ and ReLU activation, \begin{align*}
H(\mbf{a}) = \mathrm{ReLU}(\mbf{W}_H \mbf{a} + b_H)\\
l(\mbf{a}) = \mbf{w}_l^\top H(\mbf{a}) + b_l,\\
\hat{y} = \sigma(l(\mbf{a}))
\end{align*}
where $H: \mathbb{R}^r \rightarrow \mathbb{R}^h$, $l:\mathbb{R}^h \rightarrow \mathbb{R}$ and $\sigma$ is the sigmoid function. $\mbf{W}_H \in \mathbb{R}^{h \times r}$ is a learned matrix of weights and $\mbf{w}_l \in \mathbb{R}^h$ is a learned vector of weights. $b_H \in \mathbb{R}^h$ is a learned vector of biases and $b_l \in \mathbb{R}$ is a learned bias.

For binary tasks, parameters are learned by minimizing the empirical binary cross-entropy loss
\[
\mathcal{L}_{\text{BB}} 
= \mathbb{E}\big[ \ell_{\text{BCE}}(\hat{y}(\mbf{a}), y) \big].
\]

\paragraph{Sparse Autoencoder.} 
SAEs are trained by minimizing
\[
\mathcal{L}_{\text{SAE}} 
= \mathbb{E}\big[ \|\mbf{a} - \hat{\mbf{a}}\|_2^2 \big] 
+ \lambda_{\text{SAE}} \, \mathbb{E}\big[ \|s(\mbf{a})\|_1 \big],
\]
where $\lambda_{\text{SAE}} > 0$ controls sparsity.

\paragraph{Concept Bottleneck Model.}
We parameterise CBMs with a linear concept mapping \[ \hat{\mbf{c}}_{\text{known}}(\mbf{a}) = \mbf{Q} \mbf{a}, \] with $\mbf{Q} \in \mathbb{R}^{k_{\text{known}} \times r}$. The predicted concepts are then passed to a linear task head $f$, and the final binary prediction $\hat{y}(\mbf{a})$ is obtained via a sigmoid transformation:
\[
\hat{y}(\mbf{a}) = \sigma(\mbf{w}_{\textrm{CBM}}^\top \hat{\mbf{c}}_{\text{known}}(\mbf{a}) + b_{\textrm{CBM}}),
\]

We consider CBM joint training, where concept and task losses are optimized simultaneously:
\[
\mathcal{L}_{\mathrm{joint}}
=
\mathcal{L}_{\mathrm{task}}(y,\hat{y})
+
\lambda_C \,
\mathcal{L}_{\mathrm{concept}}(\mbf{c}_{\mathrm{known}}, \hat{\mbf{c}}_{\mathrm{known}}),
\]
where
\[
\mathcal{L}_{\mathrm{concept}}
=
\frac{1}{n}
\sum_{i=1}^n
\left\|
\mbf{c}^{(i)}_{\mathrm{known}} - \hat{\mbf{c}}^{(i)}_{\mathrm{known}}
\right\|_2^2,
\]
and $\mathcal{L}_{\mathrm{task}}$ is binary cross-entropy, and $\lambda_C$ balances concept prediction and task prediction.  
Both the concept predictor and task head are updated jointly.

% We consider three training regimes.
% \begin{itemize}
%     \item \textbf{Joint training.} Concept and task losses are optimized simultaneously:
% \[
% \mathcal{L}_{\mathrm{joint}}
% =
% \mathcal{L}_{\mathrm{task}}(y,\hat{y})
% +
% \lambda_C \,
% \mathcal{L}_{\mathrm{concept}}(\mbf{c}_{\mathrm{known}}, \hat{\mbf{c}}_{\mathrm{known}}),
% \]
% where
% \[
% \mathcal{L}_{\mathrm{concept}}
% =
% \frac{1}{n}
% \sum_{i=1}^n
% \left\|
% \mbf{c}^{(i)}_{\mathrm{known}} - \hat{\mbf{c}}^{(i)}_{\mathrm{known}}
% \right\|_2^2,
% \]
% and $\mathcal{L}_{\mathrm{task}}$ is binary cross-entropy, and $\lambda_C$ balances concept prediction and task prediction.  
% Both the concept predictor and task head are updated jointly.
% \item \textbf{Sequential (two-stage) training.} 
% Stage 1: train the concept predictor using
% \[
% \mathcal{L}_{\mathrm{concept}}(\mbf{c}_{\mathrm{known}}, \hat{\mbf{c}}_{\mathrm{known}}).
% \]

% Stage 2: freeze the concept predictor and train the task head $f$using predicted concepts
% \[
% \mathcal{L}_{\mathrm{task}}(y, f(\hat{\mbf{c}}_{\mathrm{known}})).
% \]
% \item \textbf{Independent training.} 
% The concept predictor is trained as above using
% \[
% \mathcal{L}_{\mathrm{concept}}(\mbf{c}_{\mathrm{known}}, \hat{\mbf{c}}_{\mathrm{known}}).
% \]

% The task head is then trained on ground-truth concepts:
% \[
% \mathcal{L}_{\mathrm{task}}(y, f(\mbf{c}_{\mathrm{known}})).
% \]

% At test time, predictions are composed as
% $
% \hat{y} = f(\hat{\mbf{c}}_{\mathrm{known}}).
% $
% This regime prevents task gradients from influencing the learned concept representations.
% \end{itemize}

\subsection{Proof of Theorem 1}

Here we prove the following theorem:
\begin{theorem*}[Closed-form Fisher metric for one-hidden-layer binary model]
Consider the black-box model defined above with sigmoid output. Then the Fisher information metric on activation space admits the closed-form expression
\[
F_A(\mbf{a}) = p(\mbf{a})(1-p(\mbf{a}))\, g(\mbf{a}) g(\mbf{a})^\top,
\]
where $p(\mbf{a})\equiv p(y=1|\mbf{a})=\sigma(l(\mbf{a}))$ and 
\[
g(\mbf{a}) = \nabla_a l(\mbf{a}).
\]
\end{theorem*}

\begin{proof}
Consider how the black box model is trained, we have an input $\mbf{a} \in \mathbb{R}^r$ which is transformed by a hidden layer according to a function $H: \mathbb{R}^r \rightarrow \mathbb{R}^h$, defined:
\begin{equation}
    H(\mbf{a})=\text{ReLU}(\mbf{W}_H\mbf{a}+b_H),
\end{equation}
where $\mbf{W}_H \in \mathbb{R}^{h \times r}$ is a learned matrix of weights and $b_H \in \mathbb{R}^h$ is a learned vector of biases.

This layer is then transformed to a logit via
\begin{equation}
    l(\mbf{a})=\mbf{w}_l^{\top}H(\mbf{a}) + b_l,
\end{equation}
where $\mbf{w}_l \in \mathbb{R}^h$ is a learned vector of weights and $b_l \in \mathbb{R}$ is a learned bias.

This is then converted into a probability via a sigmoid:
\begin{align}
    p(y=1|\mbf{a})&=\sigma(l(\mbf{a}))\\
    p(y=0|\mbf{a})&=1-\sigma(l(\mbf{a})).
\end{align}

We want to calculate:
\begin{equation}
    F_A(\mbf{a})= \mathbb{E}_{y\sim p(y|\mbf{a})}\bigg[\nabla_\mbf{a}\log p(y|\mbf{a})\nabla_\mbf{a} \log p(y|\mbf{a})^\top\bigg]
\end{equation}
Since the model defines a Bernoulli distribution over $y \in \{0,1\}$, we have:
\begin{equation}
    \log p(y|\mbf{a}) = y\log \sigma(l(\mbf{a})) + (1-y)\log (1-\sigma(l(\mbf{a}))),
\end{equation}
hence, using the identity for sigmoid functions $\sigma'(l(\mbf{a}))=\sigma(l(\mbf{a}))(1-\sigma(l(\mbf{a})))$
\begin{equation}
    \nabla_{\mbf{a}}\log p(y|\mbf{a}) = (y-\sigma(l(\mbf{a})))\nabla_{\mbf{a}} l(\mbf{a}).
\end{equation}
Taking expectations over $y$:
\begin{align}
    F_A(\mbf{a})&= \mathbb{E}[((y-\sigma(l(\mbf{a})))^2]\nabla_\mbf{a} l(\mbf{a}) \nabla_\mbf{a} l(\mbf{a})^\top\\
    &=\sigma(l(\mbf{a}))(1-\sigma(l(\mbf{a})))\nabla_\mbf{a} l(\mbf{a}) \nabla_\mbf{a} l(\mbf{a})^\top
\end{align}
which follows as $\mathbb{E}[((y-\sigma(l(\mbf{a})))^2]=\sigma(l(\mbf{a}))(1-\sigma(l(\mbf{a})))$ is the variance of the Bernoulli distribution $\text{Ber}(\sigma(l(\mbf{a})))$.

Recall
\begin{equation}
    l(\mbf{a})=\mbf{w}_l^{\top}\text{ReLU}(\mbf{W}_H\mbf{a}+b_H)+b_l,
\end{equation}
noting that $\frac{\partial \text{ReLU}(\mbf{z})}{\partial z_j}= \mathds{1} \{z_j>0\}$, and defining
\begin{equation}
    m_i(\mbf{a})= \mathds{1}\{(\mbf{W}_H \mbf{a} + b_H)_i >0\},
\end{equation}
we deduce
\begin{align}
    \frac{\partial l(\mbf{a})}{\partial a_i} = \sum_{j=1}^{h} (\mbf{w}_l)_jm_j(\mbf{a})(\mbf{W}_H)_{ji}
\end{align}
or
\begin{equation}
    \nabla_\mbf{a} l(\mbf{a}) = \mbf{W}_H^\top(m(\mbf{a}) \odot \mbf{w}_l),
\end{equation}
where $(\mbf{a} \odot \mbf{b})_j = a_jb_j$ is the Hadamard product.

Thus,
\begin{equation}
    F_A(\mbf{a})=p(\mbf{a})(1-p(\mbf{a}))g(\mbf{a})g(\mbf{a})^\top
\end{equation}
where $p(\mbf{a})=\sigma(l(\mbf{a}))$ and $g(\mbf{a})=\mbf{W}_H^\top(m(\mbf{a}) \odot \mbf{w}_l)$.
\end{proof}

\subsection{Globe treasure hunter data generation and training protocol}

\paragraph{Geometric settings.}
We consider two geometric settings: a cylindrical (``flatworld'') geometry and a spherical (``round-world'') geometry. In both settings, points $\mbf{p} \in \mathbb{R}^3$ represent possible treasure locations.

\medskip
\noindent
\textbf{Spherical geometry.}
Let $\mathbb{S}^2 = \{\mbf{u} \in \mathbb{R}^3 : \|\mbf{u}\|_2 = 1\}$ denote the unit sphere. To generate a point:

\begin{enumerate}
    \item Sample a surface direction $\mbf{u} \sim \mathrm{Unif}(\mathbb{S}^2)$.
    \item Sample a radial depth $d \sim \mathrm{Unif}([0,1))$.
    \item Define the final location as
    \[
    \mbf{p} = (1-d)\mbf{u}.
    \]
\end{enumerate}

Thus $d=0$ corresponds to the surface of the unit sphere, and increasing $d$ moves radially inward to a sphere of radius $1-d$.

\medskip
\noindent
\textbf{Cylindrical geometry.}
Let $\mathbb{D}^2 = \{(x,y)\in\mathbb{R}^2 : x^2+y^2 \leq 1\}$ denote the unit disk. To generate a point:

\begin{enumerate}
    \item Sample a surface location $(x_u,y_u) \sim \mathrm{Unif}(\mathbb{D}^2)$.
    \item Sample a vertical depth $d \sim \mathrm{Unif}([0,1))$.
    \item Define
    \[
    \mbf{p} = (x_u, y_u, -d).
    \]
\end{enumerate}

In this setting, the Earth's surface corresponds to the plane $z=0$, and increasing $d$ translates orthogonally away from the plane without altering its intrinsic disk geometry. The North and South poles correspond to the loci $y=1$ and $y=-1$ respectively.

\paragraph{Concept definitions.}

In both geometries we define three concepts. Let $\rho = 1-d$ denote the radius of the spherical cross-section at depth $d$.

\medskip
\noindent
\textbf{Supervised concepts $(C_1,C_2)$.}

We define $C_1$ and $C_2$ as surface-parallel distances to the North and South poles.

\begin{itemize}
    \item In the spherical geometry, these are geodesic distances on the sphere of radius $\rho=1-d$:
    \begin{align}
        C_1(\mbf{p}) &= \rho \cdot \arccos\!\left(\frac{\mbf{p} \cdot (0,\rho,0)}{\rho^2}\right), \\
        C_2(\mbf{p}) &= \rho \cdot \arccos\!\left(\frac{\mbf{p} \cdot (0,-\rho,0)}{\rho^2}\right).
    \end{align}
    Since $\mbf{p} = \rho \mbf{u}$ with $\mbf{u} \in \mathbb{S}^2$, these simplify to
    \[
    C_1(\mbf{p}) = \rho \arccos(u_y),
    \qquad
    C_2(\mbf{p}) = \rho \arccos(-u_y).
    \]

    \item In the cylindrical geometry, these are intrinsic (planar) distances within the disk at fixed depth:
    \begin{align}
        C_1(\mbf{p}) &= \sqrt{x^2 + (y-1)^2}, \\
        C_2(\mbf{p}) &= \sqrt{x^2 + (y+1)^2}.
    \end{align}
\end{itemize}

\medskip
\noindent
\textbf{Unsupervised concept.}
\begin{align}
    C_3(\mbf{p}) = d \quad \text{(depth)}.
\end{align}

Only $(C_1,C_2)$ are treated as supervised concepts. Depth $C_3$ is unobserved by the CBM.

\paragraph{Concept-to-activation mapping.}

For each location $\mbf{p}$, we define the concept vector
\[
\mbf{c}^{(\mbf{p})} = (C_1(\mbf{p}), C_2(\mbf{p}), C_3(\mbf{p}))^\top \in \mathbb{R}^3.
\]

Activations are generated via a linear–Gaussian model:
\begin{equation}
    \mbf{a}^{(\mbf{p})} = \bsym{\Phi} \mbf{c}^{(\mbf{p})} + \bsym{\epsilon},
\end{equation}
where:
\begin{itemize}
    \item $\bsym{\Phi} \in \mathbb{R}^{r \times 3}$ has entries independently sampled from $\mathcal{N}(0,1)$,
    \item $\bsym{\epsilon} \sim \mathcal{N}(0,\sigma_a^2 \mbf{I}_r)$,
    \item $r=64$ denotes the activation dimension.
\end{itemize}

Thus the activation space is a noisy linear embedding of the three underlying concepts.

\paragraph{Treasure hunter task.}

We fix a reference point on the equator,
\[
\mbf{e} = (1,0,0),
\]
and define the binary classification task
\begin{equation}
    Y(\mbf{p}) = \mathds{1}\{\|\mbf{p} - \mbf{e}\|_2 < R\},
\end{equation}
with $R=0.75$.

\paragraph{Dataset construction.}

For each geometry, we generate $n=8000$ samples. Data are  split into training and test sets in a 75:25 ratio. All reported results are averaged over 50 independent repetitions, each with independently sampled $\bsym{\Phi}$ and random train/test splits.

\paragraph{Model training.}

We train:

\begin{itemize}
    \item A one-hidden-layer black-box model with width $h=128$ using binary cross-entropy loss.
    \item A sparse autoencoder with $k_{\mathrm{SAE}}=60$ latent units and sparsity coefficient $\lambda_{\mathrm{SAE}}=10^{-3}$.
    \item A jointly trained CBM with concept loss weight $\lambda_C=1$.
\end{itemize}

All models are trained using the Adam optimiser \cite{kingma2017adammethodstochasticoptimization} with batch size $512$. 
The black-box model and CBM are trained with learning rate $10^{-3}$ for $30$ epochs. The SAE is trained with learning rate $2\times10^{-3}$ for 
$60$ epochs, as unsupervised feature discovery typically requires longer training.

\paragraph{Fisher geometry.}

The Fisher metric is computed using the closed-form expression derived in Theorem~1. To obtain a stable task-aligned quadratic form, we average over activations satisfying
\[
p(\mbf{a}) \in (p_{\mathrm{low}}, p_{\mathrm{high}}) \footnote{Here, $p$ represents probability, not the location vector $\mbf{p}$.},
\]
yielding $\bar{F}_A$ as defined in Equation~\eqref{ave_Fisher}.

We examined several intervals for $(p_{\mathrm{low}},p_{\mathrm{high}})\in \{(0,1),(0.1,0.9),(0.2,0.8),(0.3,0.7),(0.4,0.6)\}$ and investigated their impact on the difference between $\gamma_{F_A}$ in the cylindrical vs the spherical geometries (Figure \ref{fig:interval}). We found that the results were qualitatively comparable across different choices of interval, but quantitatively strongest   $(p_{\mathrm{low}},p_{\mathrm{high}})=(0.2,0.8)$, so we selected this interval for all experiments in the manuscript for consistency.

\begin{figure}[H]
    \centering
    \includegraphics[width=\linewidth]{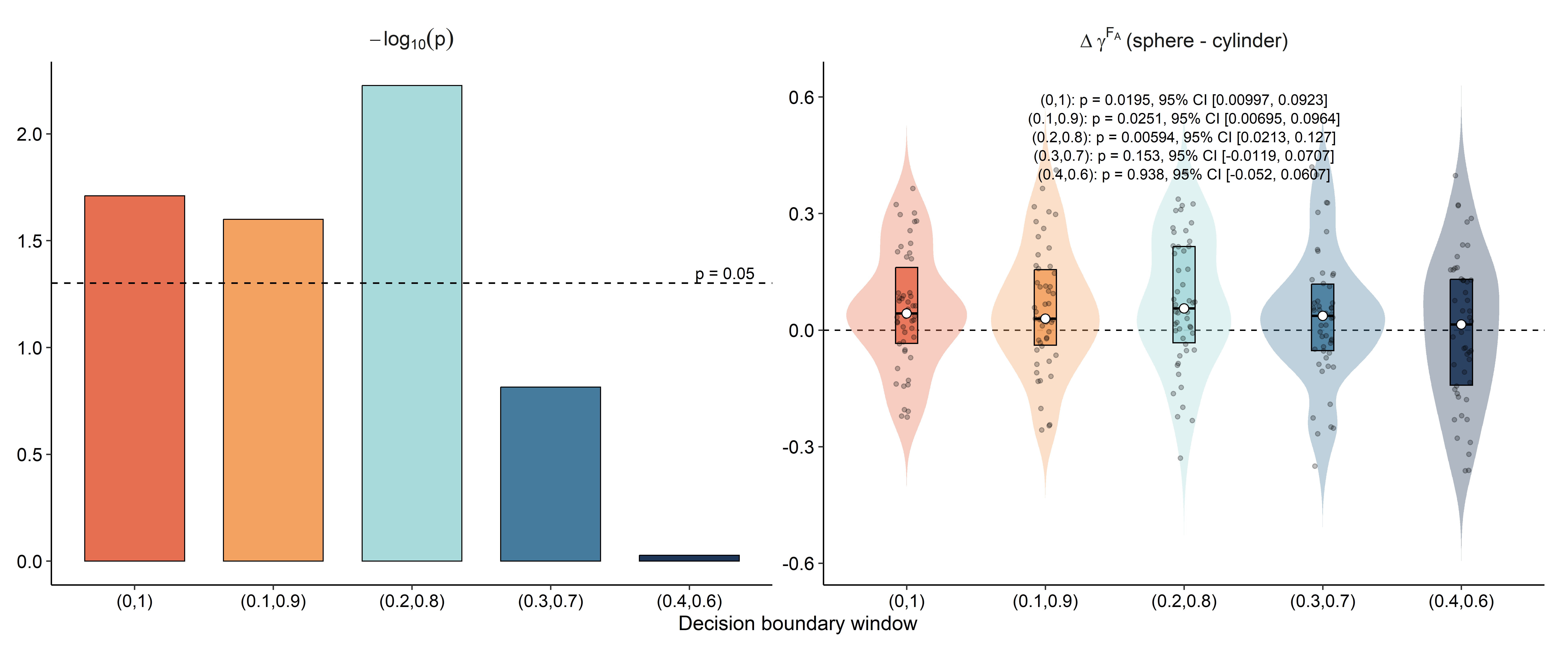}
    \vspace{-0.5em}
    \caption{\textbf{Effect of decision-boundary window size on frustration detection.} Left: Bar plot shows $-\log_{10}(p)$ for the corresponding paired Wilcoxon tests comparing $\gamma^{F_A}$ between spherical and cylindrical geometries. Right: Violin and box plots show the paired differences $\Delta \gamma^{F_A} = \gamma^{F_A}_{\mathrm{sphere}} - \gamma^{F_A}_{\mathrm{cylinder}}$ across 50 simulations for different decision-boundary windows $(p_{\mathrm{low}}, p_{\mathrm{high}})$ used to compute the Fisher metric. Across all boundary windows, $\gamma^{F_A}$ is consistently higher in the spherical geometry than in the cylindrical geometry, indicating robust detection of concept frustration. The effect is quantitatively strongest for the intermediate window $(p_{\mathrm{low}},p_{\mathrm{high}})=(0.2,0.8)$.}
    \label{fig:interval}
    \vspace{-0.5em}
\end{figure}

\paragraph{Statistical testing.}

For each metric, comparisons between cylindrical and spherical geometries are performed using a Wilcoxon signed-rank test across 50 repetitions.

\subsection{Synthetic data generation and and construction of the structured interaction matrix $\mbf{M}(\alpha)$}

In Section \ref{subsec_data_model}, we generated synthetic data from a linear-Gaussian concept model with controllable interactions between supervised and unsupervised concepts. We denote the full concept vector by
\[
\bsym{\chi}=(\bsym{\chi}_k,\bsym{\chi}_u)\in\mathbb{R}^k,
\]
where $\bsym{\chi}_k\in\mathbb{R}^{k_{\mathrm{known}}}$ contains the supervised (known) concepts and $\bsym{\chi}_u\in\mathbb{R}^{k_{\mathrm{unknown}}}$ the unsupervised (unknown) concepts, with $k=k_{\mathrm{known}}+k_{\mathrm{unknown}}$.

We first sampled a fixed symmetric positive definite covariance matrix
\[
\mbf{B}_{\mathrm{known}}\in\mathbb{R}^{k_{\mathrm{known}}\times k_{\mathrm{known}}},
\]
which defines the covariance structure among known concepts, and a fixed symmetric positive definite base covariance
\[
\mbf{B}_{\mathrm{temp}}\in\mathbb{R}^{k_{\mathrm{unknown}}\times k_{\mathrm{unknown}}},
\]
which captures the component of unknown-concept variation intrinsic to the unknown concepts. In the notation of the Results, $\mbf{B}_{kk}\equiv \mbf{B}_{\mathrm{known}}$.

To control semantic consistency versus frustration between known and unknown concepts, we constructed an $\alpha$-dependent interaction matrix
\[
\mbf{M}(\alpha)\in\mathbb{R}^{k_{\mathrm{known}}\times k_{\mathrm{unknown}}},
\]
where the sign of $\alpha$ selects whether unknown concepts reinforce ($\alpha<0$) or oppose ($\alpha>0$) the correlation structure among known concepts, and $|\alpha|$ controls interaction strength.

Let
\[
b_{ij}=(\mbf{B}_{\mathrm{known}})_{ij},
\qquad 1\le i<j\le k_{\mathrm{known}}.
\]
Each unknown concept index $u\in\{1,\dots,k_{\mathrm{unknown}}\}$ was assigned to mediate one selected known pair $(i(u),j(u))$. This assignment rule was fixed across all values of $\alpha$, and all remaining entries of the corresponding column were set to zero. For each such $u$, we set two nonzero entries in column $u$ of $\mbf{M}(\alpha)$ as follows.

\[
M(\alpha)_{i(u),\,u} = b_{i(u)j(u)}|\alpha|, \qquad 
M(\alpha)_{j(u),\,u} = -\,\mathrm{sign}(\alpha)\cdot|b_{i(u)j(u)}|\cdot|\alpha|.
\]

\iffalse
For $\alpha>0$ (frustrating regime), the two nonzero entries in column $u$ of $\mbf{M}(\alpha)$ were defined by
\[
M(\alpha)_{i(u),u}=b_{i(u)j(u)}|\alpha|,
\]
together with
\[
M(\alpha)_{j(u),u}=
\begin{cases}
-b_{i(u)j(u)}|\alpha|, & b_{i(u)j(u)}>0,\\[4pt]
\phantom{-}b_{i(u)j(u)}|\alpha|, & b_{i(u)j(u)}<0.
\end{cases}
\]

For $\alpha<0$ (consistent regime), we instead set
\[
M(\alpha)_{i(u),u}=b_{i(u)j(u)}|\alpha|,
\]
and
\[
M(\alpha)_{j(u),u}=
\begin{cases}
\phantom{-}b_{i(u)j(u)}|\alpha|, & b_{i(u)j(u)}>0,\\[4pt]
-b_{i(u)j(u)}|\alpha|, & b_{i(u)j(u)}<0.
\end{cases}
\]
\fi

At $\alpha=0$, this construction gives $\mbf{M}(0)=0$, so known and unknown concepts are independent. In the notation of the Results, we identify
\[
\mbf{B}_{ku}(\alpha)\equiv \mbf{M}(\alpha),
\qquad
\mbf{B}_{uk}(\alpha)\equiv \mbf{M}(\alpha)^\top.
\]

Given $\mbf{M}(\alpha)$, we defined the unknown-concept covariance as
\[
\mbf{B}_{uu}(\alpha)
=
\mbf{B}_{\mathrm{temp}}
+
\mbf{M}(\alpha)^\top
\mbf{B}_{\mathrm{known}}^{-1}
\mbf{M}(\alpha),
\]
and assembled the full concept covariance matrix as
\begin{equation}
\mbf{B}(\alpha)=
\begin{pmatrix}
\mbf{B}_{\mathrm{known}} & \mbf{M}(\alpha)\\
\mbf{M}(\alpha)^\top & \mbf{B}_{\mathrm{temp}}+\mbf{M}(\alpha)^\top \mbf{B}_{\mathrm{known}}^{-1}\mbf{M}(\alpha)
\end{pmatrix}.
\label{eq:toy_covariance_alpha_dependent}
\end{equation}
By Schur complement, $\mbf{B}(\alpha)$ is symmetric positive definite for all $\alpha$. Moreover,
\[
\mathrm{Cov}(\bsym{\chi}_k)=\mbf{B}_{\mathrm{known}}
\]
is preserved across regimes, so varying $\alpha$ changes only the coupling between known and unknown concepts.

We then generated $n$ independent concept samples according to
\[
\bsym{\chi}^{(i)}\sim\mathcal{N}(0,\mbf{B}(\alpha)),
\qquad i=1,\dots,n.
\]

Ground-truth task labels were generated from concepts through a noisy linear map,
\[
\mbf{a}^{(i)}=\bsym{\Phi}\bsym{\chi}^{(i)}+\bsym{\epsilon}^{(i)},
\]
where $\bsym{\Phi}\in\mathbb{R}^{r\times k}$ has entries independently sampled from $\mathcal{N}(0,1)$ and
\[
\bsym{\epsilon}^{(i)}\sim\mathcal{N}(0,\sigma_a^2\,\mbf{I}_r).
\]
The parameter $\sigma_a^2$ therefore controls the difficulty of recovering concepts from activations.

To construct the task, we first sampled a base task vector
\[
\bsym{\psi}^\ast\sim\mathcal{N}(0,\mbf{I}_k),
\]
and defined
\[
\bsym{\psi}
=
\big(
(1-\omega)\bsym{\psi}^\ast_{1:k_{\mathrm{known}}},
\;
\omega\,\bsym{\psi}^\ast_{k_{\mathrm{known}}+1:k}
\big),
\]
with $\omega\in[0,1]$. Writing $\bsym{\psi}=(\bsym{\psi}_k,\bsym{\psi}_u)$, the parameter $\omega$ controls the extent to which the task depends on known versus unknown concepts: $\omega=0$ yields a task depending only on known concepts, whereas $\omega=1$ yields a task depending only on unknown concepts.

Finally, we defined the latent task score
\[
\tau^{(i)}=\bsym{\psi}_k^\top \bsym{\chi}_k^{(i)}+\bsym{\psi}_u^\top \bsym{\chi}_u^{(i)}+\eta^{(i)},
\qquad
\eta^{(i)}\sim\mathcal{N}(0,\sigma_y^2),
\]
and generated binary labels as
\[
y^{(i)}=\mathds{1}\{\tau^{(i)}>0\}.
\]
The parameter $\sigma_y^2$ controls the intrinsic noise of the task and therefore the difficulty of predicting $y$ even from the full concept set.

\subsection{Proof of Theorem 2}

Here we prove the following theorem:

\begin{theorem*}[Concept-optimal Accuracy in the Linear–Gaussian Model]
\label{thm:optimal_cbm_accuracy}

Let the concept vector be denoted $\bsym{\chi}=(\bsym{\chi}_k,\bsym{\chi}_u)$ where $\bsym{\chi}_{k}=(c_1,\cdots,c_{k_{\mathrm{known}}})$ denotes known concepts and $\bsym{\chi}_{u}=(c_{k_{\mathrm{known}}+1},\cdots,c_{k})$ unknown concepts, under our model $\bsym{\chi} = (\bsym{\chi}_k,\bsym{\chi}_u) \sim \mathcal{N}(0,\mbf{B})$ with block covariance
\[
\mbf{B} =
\begin{pmatrix}
\mbf{B}_{kk} & \mbf{B}_{ku} \\
\mbf{B}_{uk} & \mbf{B}_{uu}
\end{pmatrix},
\]

where
\[
\mbf{B}_{uu}(\alpha)
=
\mbf{B}_{temp} + \mbf{B}_{uk}(\alpha) \mbf{B}_{kk}^{-1} \mbf{B}_{ku}(\alpha),
\]
and $\mbf{B}_{kk}$ is positive definite (and thus invertible).

Let the weight vector be denoted $\bsym{\psi} =(\bsym{\psi}_k,\bsym{\psi}_u)$ where $\bsym{\psi}_{k}=(w_1,\cdots,w_{k_{\mathrm{known}}})$ denotes known concept-task weights and $\bsym{\psi}_{u}=(w_{k_{\mathrm{known}}+1},\cdots,w_{k})$ unknown concept-task weights.

Let the task variable be
\[
\tau = \bsym{\psi}_k^\top \bsym{\chi}_k + \bsym{\psi}_u^\top \bsym{\chi}_u + \eta,
\qquad
\eta \sim \mathcal{N}(0,\sigma_y^2),
\]
and define the binary label
\[
y = \mathds{1}\{\tau > 0\}.
\]

Then the  concept-optimal Bayes classifier using only the known concepts $\bsym{\chi}_k$,
which minimises the risk
\[ R( \hat y) = \mathbb{E}_{(\bsym{\chi}_k,y)} \left[\mathds{1}\{ \hat{y}(\bsym{\chi}_k) \neq y \} \right],\]
is 
\[
\hat{y}(\bsym{\chi}_k)
=
\mathds{1}\{\bsym{\phi}(\alpha)^\top \bsym{\chi}_k > 0\},
\]
and its classification accuracy is
\[
\mathrm{Acc}_{\textrm{CBM}}(\alpha)
=
\frac{1}{2}
+
\frac{1}{\pi}
\arctan
\left(
\sqrt{
\frac{
\bsym{\phi}(\alpha)^\top \mbf{B}_{kk} \bsym{\phi}(\alpha)
}{
\bsym{\psi}_u^\top \mbf{B}_{temp} \bsym{\psi}_u + \sigma_y^2
}
}
\right),
\]
where
\[
\bsym{\phi}(\alpha)
=
\bsym{\psi}_k + \mbf{B}_{kk}^{-1} \mbf{B}_{ku}(\alpha) \bsym{\psi}_u,
\]

\end{theorem*}

\begin{proof}
    To compute the optimal classifier given known concepts, we must calculate the conditional distribution $\tau|\bsym{\chi}_k$.

We note that as our toy model is linear-Gaussian this conditional distribution will also be Gaussian and can thus be explicitly calculated via $\mathbb{E}[\tau|\bsym{\chi}_k]$ and $\mathrm{Var}(\tau|\bsym{\chi}_k)$. 

We first note that
\begin{align*}
    \mathbb{E}[\tau|\bsym{\chi}_k]&=\mathbb{E}[\bsym{\psi}_k^\top\bsym{\chi}_k+\bsym{\psi}_u^\top\bsym{\chi}_u+\eta|\bsym{\chi}_k]\\
    &=\bsym{\psi}_k^\top\bsym{\chi}_k+\bsym{\psi}_u^\top\mathbb{E}[\bsym{\chi}_u|\bsym{\chi}_k]\\\
    &=\bsym{\psi}_k^\top\bsym{\chi}_k + \bsym{\psi}_u^\top \mbf{B}_{uk}\mbf{B}^{-1}_{kk}\bsym{\chi}_k\\
    &=(\bsym{\psi}_k+\mbf{B}^{-1}_{kk}\mbf{B}_{ku}\bsym{\psi}_u)^{\top}\bsym{\chi}_k\\
    &=\bsym{\phi}^\top\bsym{\chi}_k.
\end{align*}
Which follows from linearity of expectation, the fact that $\eta\sim \mathcal{N}(0,\sigma^2_y)$, the fact that $\bsym{\chi}\sim\mathcal{N}(0,\mbf{B})$ and the conditional distribution of multivariate Gaussian distributions, and where we have defined $\bsym{\phi}=\bsym{\psi}_k+\mbf{B}^{-1}_{kk}\mbf{B}_{ku}\bsym{\psi}_u$. We note that $\bsym{\phi}:=\bsym{\phi}(\alpha)$ is a function of our frustration parameter $\alpha$.

We now consider
\begin{align*}
    \text{Var}[\tau|\bsym{\chi}_k]&=\text{Var}[\bsym{\psi}_k^\top\bsym{\chi}_k+\bsym{\psi}_u^\top\bsym{\chi}_u+\eta|\bsym{\chi}_k]\\
    &=\bsym{\psi}_u^{\top}\text{Var}[\bsym{\chi}_u|\bsym{\chi}_k]\bsym{\psi}_u+\sigma_y^2\\
    &=\bsym{\psi}_u^{\top}(\mbf{B}_{uu}-\mbf{B}_{uk}\mbf{B}^{-1}_{kk}\mbf{B}_{ku})\bsym{\psi}_u+\sigma_y^2\\
    &=\bsym{\psi}_u^{\top}\mbf{B}_{\text{temp}}\bsym{\psi}_u+\sigma_y^2.
\end{align*}
Which follows from the variance of a constant being zero the fact that $\eta\sim \mathcal{N}(0,\sigma^2_y)$ and is an independent variable, the fact that $\bsym{\chi}\sim\mathcal{N}(0, \mbf{B})$, the conditional distribution of multivariate Gaussian distributions, and the construction of $\mbf{B}$ which ensures $\mbf{B}_{\text{temp}}=\mbf{B}_{uu}-\mbf{B}_{uk}\mbf{B}^{-1}_{kk}\mbf{B}_{ku}$. We note that via our construction this variance is not dependent on $\alpha$, the frustration parameter.

Whence $\tau|\bsym{\chi}_k\sim \mathcal{N}(\bsym{\phi}(\alpha)^\top\bsym{\chi}_k,\bsym{\psi}_u^{\top}\mbf{B}_{\text{temp}}\bsym{\psi}_u+\sigma_y^2)$. 
By definition the concept-optimal Bayes predictor is the risk minimiser 
\[ R( \hat y) = \mathbb{E}_{(\mbf{c}_{\text{known}},y)} \left[\ell(\hat{y}(\mbf{c}_{\text{known}}),  y) \right],\]
with respect to the 0-1 loss.
Since $ \mathbb{E} \left[\ell(\hat{y}(\mbf{c}_{\text{known}}),  y)  \right] = 
 \mathbb{E} \left [ \mathbb{E} \left[\ell(\hat{y}(\mbf{c}_{\text{known}}),  y)  \mid \mbf{c}_{\text{known}}\right] \right ]$,
 which is minimised by setting
 \[ \hat y(\mbf{c}_{\text{known}}) = \mathrm{argmin}_{z \in \{0,1\}} \mathbb{E} \left[\ell(z,  y)  \mid \mbf{c}_{\text{known}}\right]. \]
 When $l(z,y) $ is the 0-1 loss, 
 we obtain as usual 
$  \hat y( \mbf{c}_{\text{known}})=\mathrm{argmax}_{z} \mathbb{P}(Y=z | \chi_k = \mbf{c}_{\text{known}})$.
Hence a concept-optimal model will predict $y=1$ 
iff  
$\mathbb{P}(Y=1 | \chi_k = \mbf{c}_{\text{known}}) > 1/2$,
and thus
will predict $y=1$ when $\mathbb{P}[\tau>0|\mbf{c}_{\text{known}}]>1/2$ and zero otherwise. 
It therefore follows that the optimal classifier given $\bsym{\chi}_k$ is simply:
\begin{equation}
    \hat{y}(\bsym{\chi}_k)= \mathds{1}\{\bsym{\phi}(\alpha)^\top\bsym{\chi}_k>0\}.
\end{equation}

As the true outcome is $y=\mathds{1}\{\tau>0\}$ We note that the optimal classifier makes an error when $\textrm{sign}(\bsym{\phi}(\alpha)^\top\bsym{\chi}_k) \not=\textrm{sign}(\tau)$.

Let us define $S:=\mathbb{E}[\tau |\bsym{\chi}_k]=\bsym{\phi}(\alpha)^\top\bsym{\chi}_k$ and $R:=\tau-S$. We note that as $\bsym{\chi}_k\sim\mathcal{N}(0,\mbf{B}_{kk})$, it follows that $S \sim \mathcal{N}(0,\bsym{\phi}(\alpha)^\top\mbf{B}_{kk}\bsym{\phi}(\alpha))$, we denote $\sigma_S^2=\bsym{\phi}(\alpha)^\top\mbf{B}_{kk}\bsym{\phi}(\alpha)$. We also note that
\begin{align*}
    R&=\tau-S\\
    &= \bsym{\psi}_k^\top\bsym{\chi}_k+\bsym{\psi}_u^\top\chi_u+\eta - (\bsym{\psi}_k+\mbf{B}^{-1}_{kk}\mbf{B}_{ku}\bsym{\psi}_u)^\top\bsym{\chi}_k\\
    &=\bsym{\psi}_u^\top(\bsym{\chi}_u-\mbf{B}_{uk}\mbf{B}_{kk}^{-1}\bsym{\chi}_k) + \eta.
\end{align*}
It is clear that $\mathbb{E}[R]=0$. We note that 
\begin{align*}
    \text{Var}[R]&=\mathbb{E}[(\bsym{\psi}_u^\top(\bsym{\chi}_u-\mbf{B}_{uk}\mbf{B}_{kk}^{-1}\bsym{\chi}_k))^{2}]+2\mathbb{E}[\eta]\mathbb{E}[\bsym{\psi}_u^\top(\bsym{\chi}_u-\mbf{B}_{uk}\mbf{B}_{kk}^{-1}\bsym{\chi}_k)]] + \mathbb{E}[\eta^2]\\
    &=\mathbb{E}[(\bsym{\psi}_u^\top(\bsym{\chi}_u-\mbf{B}_{uk}\mbf{B}_{kk}^{-1}\bsym{\chi}_k))^{2}] + \sigma_y^2\\
    &=\bsym{\psi}_u^\top\mathbb{E}[(\bsym{\chi}_u-\mbf{B}_{uk}\mbf{B}_{kk}^{-1}\bsym{\chi}_k)^{2}]\bsym{\psi}_u+\sigma_y^2\\
    &=\bsym{\psi}_u^\top\mathbb{E}[(\bsym{\chi}_u-\mathbb{E}[\bsym{\chi}_u|\bsym{\chi}_k])^{2}]\bsym{\psi}_u+\sigma_y^2\\
    &=\bsym{\psi}_u^\top\mathrm{Var}[\bsym{\chi}_u \mid \bsym{\chi}_k]\bsym{\psi}_u+\sigma_y^2\\
    &=\bsym{\psi}_u^\top(\mbf{B}_{uu}-\mbf{B}_{uk}B^{-1}_{kk}\mbf{B}_{ku})\bsym{\psi}_u+\sigma_y^2,
\end{align*}
which follows from the properties of a multivariate Gaussian. Thus $R\sim\mathcal{N}(0,\bsym{\psi}_u^\top(\mbf{B}_{uu}-\mbf{B}_{uk}\mbf{B}^{-1}_{kk}\mbf{B}_{ku})\bsym{\psi}_u+\sigma_y^2)$, we denote $\sigma_R^2=\bsym{\psi}_u^\top(\mbf{B}_{uu}-\mbf{B}_{uk}\mbf{B}^{-1}_{kk}\mbf{B}_{ku})\bsym{\psi}_u+\sigma_y^2$.

Thus the accuracy of the Bayes optimal classifier is $\mathbb{P}[\textrm{sign}(S)=\textrm{sign}(S+R)]$. We note that $S$ and $R$ are independent Gaussian variables (i.e., $\textrm{Cov}(R,S)=0$) and that $\textrm{Var}(S+R)=\sigma_R^2+\sigma_S^2$ and $\textrm{Cov}(S,S+R)=\sigma_S^2$, thus the correlation co-efficient between $S+R$ and $S$, is $\rho_{S,S+R}=\frac{\sigma_S}{\sqrt{\sigma_S^2+\sigma_R^2}}$. It can therefore be demonstrated through simple integration that the accuracy of the Bayes-optimal CBM given $\bsym{\chi}_k$ is given by:
\begin{align*}
    \mathbb{P}[\textrm{sign}(S)&=\textrm{sign}(S+R)] = \frac{1}{2}+\frac{1}{\pi}\arcsin \rho_{S,S+R}\\
    &=\frac{1}{2}+\frac{1}{\pi}\arcsin \Bigg(\frac{\sigma_S}{\sqrt{\sigma_S^2+\sigma_R^2}}\Bigg)\\
    &=\frac{1}{2}+\frac{1}{\pi}\arctan \Bigg(\frac{\sigma_S}{\sigma_R}\Bigg)\\
    &=\frac{1}{2}+\frac{1}{\pi}\arctan \Bigg(\sqrt{\frac{\bsym{\phi}(\alpha)^\top\mbf{B}_{kk}\bsym{\phi}(\alpha)}{\bsym{\psi}_u^\top(\mbf{B}_{uu}-\mbf{B}_{uk}\mbf{B}^{-1}_{kk}\mbf{B}_{ku})\bsym{\psi}_u+\sigma_y^2}}\Bigg)\\
    &=\frac{1}{2}+\frac{1}{\pi}\arctan \Bigg(\sqrt{\frac{\bsym{\phi}(\alpha)^\top\mbf{B}_{kk}\bsym{\phi}(\alpha)}{\bsym{\psi}_u^\top\mbf{B}_{temp}\bsym{\psi}_u+\sigma_y^2}}\Bigg)
\end{align*}

We denote this quantity as $\textrm{Acc}_{\textrm{CBM}}(\alpha)$.
\end{proof}

\subsection{Experiments with synthetic data}

We used the linear-Gaussian synthetic data generator defined above to study interactions between supervised (``known'') and unsupervised (``unknown'') concepts under controlled covariance structure.

For each simulation we generated:

\begin{align*}
n &= 8000, \\
k &= 50, \\
k_{\mathrm{known}} &\in \{10,20,30,40\}, \\
r &= 64, \\
\sigma_a &= 0.3, \\
\sigma_y &= 1.5.
\end{align*}

Concepts were sampled as
\[
\mbf{c} = (\mbf{c}_{\mathrm{known}}, \mbf{c}_{\mathrm{unknown}}) \sim \mathcal{N}(0, \mbf{B}(\alpha)),
\]
where $\alpha \in \{-1,0,1\}$ controls semantic alignment between known and unknown concepts:

\begin{itemize}
\item $\alpha=-1$: semantic alignment (consistent)
\item $\alpha=0$: independence
\item $\alpha=+1$: semantic opposition (frustration)
\end{itemize}

The outcome was generated as
\[
Y = \bsym{\psi}_k^\top \mbf{c}_{\mathrm{known}} + \omega \bsym{\psi}_u^\top \mbf{c}_{\mathrm{unknown}} + \varepsilon,
\]
where $\omega \in \{0,0.25,0.5,0.75,1\}$ controls the proportion of task signal arising from unknown concepts and $\varepsilon \sim \mathcal{N}(0,\sigma_y^2)$.

Observed features were generated via a linear projection
\[
X = \mbf{A}\mbf{c} + \eta,
\]
where $\eta \sim \mathcal{N}(0,\sigma_a^2 I)$.

Each of the 60 parameter regimes was simulated across 10 independent seeds (600 datasets total). Data were split 75:25 into train and test sets.

% \subsubsection{Evaluation metrics}

% \paragraph{Task accuracy.}
% Test-set classification accuracy was computed for all models on held out data. Bayes-optimal accuracy was computed analytically under the linear-Gaussian data-generating model.

% \paragraph{Concept prediction error.}
% Concept mean squared error (MSE) was computed over held-out data.

% \paragraph{Frustration metrics.}
% We computed both task-aligned Fisher frustration $\gamma^{F_A}$ and Euclidean frustration $\gamma^{E}$ as described in Results.

% \paragraph{Statistical analysis.} For seed-matched regimes, paired Wilcoxon signed-rank tests were used to compare $\alpha=-1$,\ $\alpha=+1$ and $\alpha=0$. Statistical significance was assessed at $p<0.05$.

\subsection{Real-world experiments}

\paragraph{Language experiment (sarcasm detection).}

\textit{Dataset and embeddings.}
We used the News Headlines sarcasm dataset (28,000 headlines; 13,000 labelled sarcastic) \cite{misra2023Sarcasm}. Each headline was embedded using DeBERTa-v3-base, producing a 768-dimensional representation $\mbf{a}$ \cite{he2021debertav3}.

\textit{Concept construction.}
We defined three continuous sentiment concepts derived from RoBERTa-base-sentiment logits \cite{barbieri-etal-2020-tweeteval}: $C_1$ (positive sentiment), $C_2$ (negative sentiment), and $C_3$ (sentiment intensity), defined as the negative neutrality logit. In this dataset, $C_1$ and $C_2$ are inversely correlated, while $C_3$ is positively correlated with both. A linear model confirmed that all three concepts are independently associated with sarcasm.

\textit{Models and training.}
Two CBMs were trained using 10-fold cross-validation: CBM1 with known concepts $(C_1, C_2)$ and CBM2 with known concepts $(C_1, C_2, C_3)$. Both models were trained for 30 epochs with learning rate $10^{-3}$ and evaluated using 10-fold cross-validation. A one-hidden-layer MLP was trained with width $h = 128$ for 60 epochs. We also trained an SAE with a learning rate of $2\times10^{-3}$ for 60 epochs with $k_{\mathrm{SAE}} = 300$ latent units and a sparsity coefficient  $\lambda_{\mathrm{SAE}}=10^{-3}$. All models were trained using the Adam optimiser with a batch size of 512.

\textit{Evaluation.}
We computed task accuracy, Euclidean frustration $\textrm{Frust}^E(1,2)$, and Fisher frustration $\textrm{Frust}^{F_A}(1,2)$. Frustration was evaluated on the shared concept subspace $(C_1, C_2)$. Paired Wilcoxon signed-rank tests were used across folds.

\vspace{0.5em}

\paragraph{Vision experiment (CUB gull--tern classification).}

\textit{Dataset and embeddings.}
We used the CUB-200-2011 dataset \cite{wah2011caltech}, restricting to gull and tern species (887 images total). Images were embedded using the pretrained CLIP ViT-B/16 encoder, yielding a 512-dimensional representation $\mbf{a}$ \cite{bhalla2024interpreting}.

\textit{Concept construction.}
We selected three visual concepts from the CUB attribute vocabulary: white upperparts ($C_1$), gull-like shape ($C_2$), and solid wing pattern ($C_3$). Within the gull–tern subset, $C_1$ and $C_2$ are positively correlated, while $C_3$ is positively correlated with $C_1$ and negatively correlated with $C_2$, forming a frustrating triple. All three concepts were independently associated with the classification task.

\textit{Models and training.}
Two CBMs were trained using 10-fold cross-validation: CBM1 with known concepts $(C_1, C_2)$ and CBM2 with known concepts $(C_1, C_2, C_3)$.  Both CBMs were trained for 30 epochs with a learning rate of $10^{-3}$ using joint optimisation of concept reconstruction loss  ($\lambda_C = 1$) and task loss. The SAE was trained with a learning rate of $2\times10^{-3}$ for 60 epochs with $k_{\mathrm{SAE}} = 300$ latent units and a sparsity coefficient  $\lambda_{\mathrm{SAE}}=10^{-3}$. A one-hidden-layer MLP was trained with width $h = 128$ for 60 epochs with a learning rate of $10^{-3}$. All models were trained using the Adam optimiser with a batch size of 512 and evaluation was performed using 10-fold stratified cross-validation.

\textit{Evaluation.}
We computed task accuracy, Euclidean frustration $\textrm{Frust}^E(1,2)$, and Fisher frustration $\textrm{Frust}^{F_A}(1,2)$, again restricted to the shared concept subspace $(C_1, C_2)$. Paired Wilcoxon signed-rank tests were used across folds.

\subsection{Code and data availability.}
Code and data to reproduce the results of this paper are available at \url{https://github.com/CRSBanerji/concept-frustration}.

\section{Acknowledgements}

EP is funded by the Department of Science Innovation and Technology under PharosAI. CS is funded by the MANIFEST programme, core funded by the UK Government’s Office for Life Sciences and the Medical Research Council (Grant Ref: MR/Z505158/1). AI is funded by the Martingale Foundation. CRSB is supported by a King's College London AI+ Fellowship. 

\section{Supplementary Data}

Table \ref{tab:geometry_comparison},
Table \ref{tab:alpha_regime_summary_full},
Table \ref{tab:cbm1_vs_cbm2},
Table \ref{tab:cub_cbm1_vs_cbm2}.

\begin{table}[htbp]
\centering
\footnotesize
\setlength{\tabcolsep}{6pt}
\renewcommand{\arraystretch}{1.25}
\begin{tabular}{lcccc}
\toprule
Metric &
\makecell{Cylinder} &
\makecell{Sphere} &
\makecell{Hodges--Lehmann\\
median paired difference\\(sphere$-$cylinder)} &
$p$ \\
\midrule

BB acc.
& \makecell{$0.876$\\$[0.874,\,0.878]$}
& \makecell{$0.894$\\$[0.891,\,0.896]$}
& \makecell{$0.0180$\\$[0.0152,\,0.0207]$}
& $\mathbf{<10^{-4}}$ \\

CBM acc.
& \makecell{$0.811$\\$[0.759,\,0.862]$}
& \makecell{$0.824$\\$[0.782,\,0.866]$}
& \makecell{$0.0132$\\$[0.00805,\,0.0177]$}
& $\mathbf{3.9\times10^{-3}}$ \\

Concept MSE
& \makecell{$2.81\mathrm{e}{-2}$\\$[2.58,\,3.05]\mathrm{e}{-2}$}
& \makecell{$2.99\mathrm{e}{-2}$\\$[2.70,\,3.29]\mathrm{e}{-2}$}
& \makecell{$1.10\mathrm{e}{-3}$\\$[-6.34,\,3.15]\mathrm{e}{-3}$}
& $0.284$ \\

$\mathbf{\gamma^{F_A}}$
& \makecell{$0.126$\\$[0.0896,\,0.163]$}
& \makecell{$0.198$\\$[0.163,\,0.233]$}
& \makecell{$7.25\mathrm{e}{-2}$\\$[0.0213,\,0.127]$}
& $\mathbf{5.94\times10^{-3}}$ \\

$\gamma^{E}$
& \makecell{$4.85\mathrm{e}{-2}$\\$[4.30,\,5.39]\mathrm{e}{-2}$}
& \makecell{$4.47\mathrm{e}{-2}$\\$[3.88,\,5.05]\mathrm{e}{-2}$}
& \makecell{$-3.52\mathrm{e}{-3}$\\$[-8.25,\,9.15]\mathrm{e}{-4}$}
& $0.118$ \\

$G$
& \makecell{$0.927$\\$[0.922,\,0.933]$}
& \makecell{$0.926$\\$[0.919,\,0.932]$}
& \makecell{$-1.87\mathrm{e}{-3}$\\$[-6.44,\,2.99]\mathrm{e}{-3}$}
& $0.469$ \\

\bottomrule
\end{tabular}
\caption{Supplementary Table S1: Comparison of performance and geometric metrics between cylindrical (flatworld) and spherical (round-world) geometries. Values are means (top) with 95\% confidence intervals (bottom). Paired Wilcoxon signed-rank tests compare matched repetitions (sphere minus cylinder).}
\label{tab:geometry_comparison}
\end{table}
%SI table for the simulation
\begin{table}[htbp]
\centering
\footnotesize
\setlength{\tabcolsep}{6pt}
\renewcommand{\arraystretch}{1.2}

%========================================================
% Top block: regime summaries + alpha = +1 vs -1
%========================================================
\begin{tabular}{lccccc}
\toprule
Metric &
\makecell{$\alpha=0$\\Independent} &
\makecell{$\alpha=-1$\\Consistent} &
\makecell{$\alpha=1$\\Frustrated} &
\makecell{Median diff\\(+1 minus $-1$)\\(95\% CI)} &
\makecell{$p$\\Paired} \\
\midrule

BB acc.
& \makecell{$0.935$\\$[0.932,\,0.938]$}
& \makecell{$0.936$\\$[0.930,\,0.943]$}
& \makecell{$0.936$\\$[0.929,\,0.944]$}
& \makecell{$2.03\mathrm{e}{-3}$\\$[-5.02\mathrm{e}{-4},\,4.77\mathrm{e}{-3}]$}
& $0.132$ \\

CBM acc.
& \makecell{$0.783$\\$[0.769,\,0.797]$}
& \makecell{$0.718$\\$[0.703,\,0.734]$}
& \makecell{$0.699$\\$[0.684,\,0.714]$}
& \makecell{$-1.63\mathrm{e}{-2}$\\$[-2.25,\,-1.03]\mathrm{e}{-2}$}
& $\mathbf{<10^{-4}}$ \\

Concept MSE
& \makecell{$2.10$\\$[2.05,\,2.15]$}
& \makecell{$40.9$\\$[26.0,\,55.8]$}
& \makecell{$63.7$\\$[43.4,\,84.1]$}
& \makecell{$6.23$\\$[4.42,\,12.5]$}
& $\mathbf{<10^{-4}}$ \\

$\beta$
& \makecell{$26.0$\\$[23.8,\,28.3]$}
& \makecell{$258$\\$[168,\,347]$}
& \makecell{$373$\\$[264,\,483]$}
& \makecell{$39.8$\\$[15.0,\,74.0]$}
& $\mathbf{<10^{-4}}$ \\

$G$
& \makecell{$0.995$\\$[0.994,\,0.996]$}
& \makecell{$0.991$\\$[0.987,\,0.995]$}
& \makecell{$0.985$\\$[0.974,\,0.996]$}
& \makecell{$-1.76\mathrm{e}{-4}$\\$[-7.08\mathrm{e}{-4},\,3.66\mathrm{e}{-4}]$}
& $0.513$ \\

$\mathbf{\gamma^{F_A}}$
& \makecell{$1.60\mathrm{e}{-2}$\\$[1.48,\,1.71]\mathrm{e}{-2}$}
& \makecell{$2.76\mathrm{e}{-2}$\\$[2.51,\,3.02]\mathrm{e}{-2}$}
& \makecell{$3.67\mathrm{e}{-2}$\\$[3.31,\,4.04]\mathrm{e}{-2}$}
& \makecell{$7.76\mathrm{e}{-3}$\\$[5.47\mathrm{e}{-3},\,1.00\mathrm{e}{-2}]$}
& $\mathbf{<10^{-4}}$ \\

$\gamma^{E}$
& \makecell{$7.62\mathrm{e}{-3}$\\$[7.38,\,7.85]\mathrm{e}{-3}$}
& \makecell{$7.43\mathrm{e}{-3}$\\$[6.96,\,7.90]\mathrm{e}{-3}$}
& \makecell{$7.99\mathrm{e}{-3}$\\$[7.37,\,8.61]\mathrm{e}{-3}$}
& \makecell{$1.51\mathrm{e}{-4}$\\$[3.28\mathrm{e}{-5},\,3.35\mathrm{e}{-4}]$}
& $0.0108$ \\

$T_1$
& \makecell{$251$\\$[195,\,307]$}
& \makecell{$251$\\$[195,\,307]$}
& \makecell{$251$\\$[195,\,307]$}
& --
& -- \\

$T_2$
& \makecell{$0$\\$[0,\,0]$}
& \makecell{$-0.327$\\$[-0.982,\,0.327]$}
& \makecell{$0.404$\\$[-0.448,\,1.26]$}
& \makecell{$0.304$\\$[-1.34,\,2.07]$}
& $0.742$ \\

$T_3$
& \makecell{$0$\\$[0,\,0]$}
& \makecell{$2.76\mathrm{e}{6}$\\$[-6.27\mathrm{e}{5},\,6.15\mathrm{e}{6}]$}
& \makecell{$3.27\mathrm{e}{5}$\\$[-6.18\mathrm{e}{4},\,7.15\mathrm{e}{5}]$}
& \makecell{$615$\\$[237,\,1.37\mathrm{e}{3}]$}
& $\mathbf{<10^{-4}}$ \\

$T_4$
& \makecell{$326$\\$[255,\,398]$}
& \makecell{$326$\\$[255,\,398]$}
& \makecell{$326$\\$[255,\,398]$}
& --
& -- \\

\bottomrule
\end{tabular}

\vspace{0.6cm}

%========================================================
% Bottom block: comparisons against alpha = 0
%========================================================
\begin{tabular}{lcccc}
\toprule
Metric &
\makecell{Median diff\\(+1 minus 0)\\(95\% CI)} &
\makecell{$p$\\0 vs +1} &
\makecell{Median diff\\($-1$ minus 0)\\(95\% CI)} &
\makecell{$p$\\0 vs $-1$} \\
\midrule

BB acc.
& \makecell{$5.24\mathrm{e}{-3}$\\$[7.40\mathrm{e}{-4},\,1.02\mathrm{e}{-2}]$}
& $0.0232$
& \makecell{$2.70\mathrm{e}{-3}$\\$[1.60\mathrm{e}{-5},\,5.98\mathrm{e}{-3}]$}
& $0.050$ \\

CBM acc.
& \makecell{$-7.72\mathrm{e}{-2}$\\$[-8.83\mathrm{e}{-2},\,-6.66\mathrm{e}{-2}]$}
& $\mathbf{<10^{-4}}$
& \makecell{$-5.83\mathrm{e}{-2}$\\$[-6.72\mathrm{e}{-2},\,-4.92\mathrm{e}{-2}]$}
& $\mathbf{<10^{-4}}$ \\

Concept MSE
& \makecell{$20.2$\\$[11.8,\,25.7]$}
& $\mathbf{<10^{-4}}$
& \makecell{$5.75$\\$[4.23,\,14.3]$}
& $\mathbf{<10^{-4}}$ \\

$\beta$
& \makecell{$136$\\$[69.9,\,176]$}
& $\mathbf{<10^{-4}}$
& \makecell{$44.3$\\$[28.6,\,79.7]$}
& $\mathbf{<10^{-4}}$ \\

$G$
& \makecell{$2.45\mathrm{e}{-4}$\\$[-4.28\mathrm{e}{-4},\,9.34\mathrm{e}{-4}]$}
& $0.443$
& \makecell{$3.68\mathrm{e}{-4}$\\$[-1.82\mathrm{e}{-4},\,8.97\mathrm{e}{-4}]$}
& $0.172$ \\

$\gamma^{F_A}$
& \makecell{$1.70\mathrm{e}{-2}$\\$[1.44\mathrm{e}{-2},\,1.99\mathrm{e}{-2}]$}
& $\mathbf{<10^{-4}}$
& \makecell{$9.29\mathrm{e}{-3}$\\$[7.30\mathrm{e}{-3},\,1.14\mathrm{e}{-2}]$}
& $\mathbf{<10^{-4}}$ \\

$\gamma^{E}$
& \makecell{$-9.00\mathrm{e}{-4}$\\$[-1.11\mathrm{e}{-3},\,-5.96\mathrm{e}{-4}]$}
& $\mathbf{<10^{-4}}$
& \makecell{$-9.59\mathrm{e}{-4}$\\$[-1.14\mathrm{e}{-3},\,-7.72\mathrm{e}{-4}]$}
& $\mathbf{<10^{-4}}$ \\

$T_1$
& --
& --
& --
& -- \\

$T_2$
& \makecell{$0.934$\\$[-0.211,\,2.02]$}
& $0.112$
& \makecell{$-0.277$\\$[-1.26,\,0.95]$}
& $0.601$ \\

$T_3$
& \makecell{$2.06\mathrm{e}{3}$\\$[1.15\mathrm{e}{3},\,3.47\mathrm{e}{3}]$}
& $\mathbf{<10^{-4}}$
& \makecell{$639$\\$[386,\,953]$}
& $\mathbf{<10^{-4}}$ \\

$T_4$
& --
& --
& --
& -- \\

\bottomrule
\end{tabular}

\caption{Supplementary Table S2: Summary statistics across concept regimes defined by $\alpha$. The upper block reports regime means with 95\% confidence intervals together with the paired comparison between the frustrated ($\alpha=1$) and consistent ($\alpha=-1$) settings. The lower block reports paired Hodges--Lehmann median shift estimates with 95\% confidence intervals and paired Wilcoxon signed-rank $p$-values for comparisons against the independent baseline ($\alpha=0$).}
\label{tab:alpha_regime_summary_full}
\end{table}

\begin{table}[htbp]
\centering
\footnotesize
\setlength{\tabcolsep}{6pt}
\renewcommand{\arraystretch}{1.25}
\begin{tabular}{lcccc}
\toprule
Metric &
\makecell{CBM1\\Mean (95\% CI)} &
\makecell{CBM2\\Mean (95\% CI)} &
\makecell{Median diff\\(CBM1 $-$ CBM2)} &
$p$ \\
\midrule

CBM acc.
& \makecell{$0.728$\\$[0.657,\,0.798]$}
& \makecell{$0.805$\\$[0.777,\,0.834]$}
& \makecell{$-7.20\mathrm{e}{-2}$\\$[-1.77\mathrm{e}{-1},\,1.04\mathrm{e}{-3}]$}
& $0.058$ \\

$\gamma^{E}$
& \makecell{$1.01\mathrm{e}{-3}$\\$[6.59\mathrm{e}{-4},\,1.36\mathrm{e}{-3}]$}
& \makecell{$1.31\mathrm{e}{-3}$\\$[8.00\mathrm{e}{-4},\,1.82\mathrm{e}{-3}]$}
& \makecell{$-2.75\mathrm{e}{-4}$\\$[-7.82\mathrm{e}{-4},\,1.31\mathrm{e}{-4}]$}
& $0.492$ \\

$\gamma^{F_A}$
& \makecell{$4.47\mathrm{e}{-2}$\\$[-1.71\mathrm{e}{-2},\,1.06\mathrm{e}{-1}]$}
& \makecell{$5.30\mathrm{e}{-3}$\\$[-1.83\mathrm{e}{-3},\,1.24\mathrm{e}{-2}]$}
& \makecell{$3.13\mathrm{e}{-3}$\\$[3.21\mathrm{e}{-4},\,1.12\mathrm{e}{-1}]$}
& $\mathbf{0.0488}$ \\

$G$
& \makecell{$0.989$\\$[0.987,\,0.991]$}
& \makecell{$0.987$\\$[0.987,\,0.988]$}
& \makecell{$1.50\mathrm{e}{-3}$\\$[-4.46\mathrm{e}{-4},\,3.55\mathrm{e}{-3}]$}
& $0.131$ \\

% $\beta$
% & \makecell{$2.16$\\$[2.04,\,2.28]$}
% & \makecell{$2.17$\\$[2.05,\,2.29]$}
% & \makecell{$-1.41\mathrm{e}{-2}$\\$[-6.05\mathrm{e}{-2},\,3.94\mathrm{e}{-2}]$}
% & $0.625$ \\

\bottomrule
\end{tabular}
\caption{Supplementray Table S3: Comparison of CBM1 and CBM2 performance metrics in the language task. Values are mean (top) with 95\% confidence intervals (bottom). Median differences are Hodges--Lehmann estimates from paired Wilcoxon signed-rank tests (CBM1 $-$ CBM2).}
\label{tab:cbm1_vs_cbm2}
\end{table}

\begin{table}[htbp]
\centering
\footnotesize
\setlength{\tabcolsep}{6pt}
\renewcommand{\arraystretch}{1.25}
\begin{tabular}{lcccc}
\toprule
Metric &
\makecell{CBM1\\Mean (95\% CI)} &
\makecell{CBM2\\Mean (95\% CI)} &
\makecell{Median diff\\(CBM1 $-$ CBM2)} &
$p$ \\
\midrule

CBM acc.
& \makecell{$0.639$\\$[0.547,\,0.731]$}
& \makecell{$0.755$\\$[0.680,\,0.829]$}
& \makecell{$-1.27\mathrm{e}{-1}$\\$[-2.62\mathrm{e}{-1},\,2.25\mathrm{e}{-2}]$}
& $0.0972$ \\

$\gamma^{E}$
& \makecell{$7.22\mathrm{e}{-3}$\\$[6.48\mathrm{e}{-3},\,7.95\mathrm{e}{-3}]$}
& \makecell{$6.98\mathrm{e}{-3}$\\$[5.87\mathrm{e}{-3},\,8.09\mathrm{e}{-3}]$}
& \makecell{$4.29\mathrm{e}{-4}$\\$[-9.77\mathrm{e}{-4},\,1.46\mathrm{e}{-3}]$}
& $0.557$ \\

$\gamma^{F_A}$
& \makecell{$6.79\mathrm{e}{-2}$\\$[5.67\mathrm{e}{-3},\,1.30\mathrm{e}{-1}]$}
& \makecell{$6.49\mathrm{e}{-3}$\\$[-3.83\mathrm{e}{-3},\,1.68\mathrm{e}{-2}]$}
& \makecell{$5.71\mathrm{e}{-2}$\\$[4.65\mathrm{e}{-5},\,1.22\mathrm{e}{-1}]$}
& $\mathbf{0.0488}$ \\

$G$
& \makecell{$0.995$\\$[0.993,\,0.997]$}
& \makecell{$0.993$\\$[0.991,\,0.995]$}
& \makecell{$1.69\mathrm{e}{-3}$\\$[-1.60\mathrm{e}{-3},\,5.22\mathrm{e}{-3}]$}
& $0.232$ \\

\bottomrule
\end{tabular}
\caption{Supplementrary Table S4: Comparison of CBM1 and CBM2 on the CUB gull--tern classification task using CLIP ViT-B/16 embeddings. Values are mean (top) with 95\% confidence intervals (bottom). Median differences are Hodges--Lehmann estimates from paired Wilcoxon signed-rank tests (CBM1 $-$ CBM2).}
\label{tab:cub_cbm1_vs_cbm2}
\end{table}

\bibliography{refs}

\begin{thebibliography}{10}

\bibitem{rudin2019stop}
Cynthia Rudin.
\newblock Stop explaining black box machine learning models for high stakes
  decisions and use interpretable models instead.
\newblock {\em Nature machine intelligence}, 1(5):206--215, 2019.

\bibitem{doshi2017towards}
Finale Doshi-Velez and Been Kim.
\newblock Towards a rigorous science of interpretable machine learning.
\newblock {\em arXiv preprint arXiv:1702.08608}, 2017.

\bibitem{hassija2024interpreting}
Vikas Hassija, Vinay Chamola, Atmesh Mahapatra, Abhinandan Singal, Divyansh
  Goel, Kaizhu Huang, Simone Scardapane, Indro Spinelli, Mufti Mahmud, and Amir
  Hussain.
\newblock Interpreting black-box models: a review on explainable artificial
  intelligence.
\newblock {\em Cognitive Computation}, 16(1):45--74, 2024.

\bibitem{EUMemberStates2024}
EU~Member States.
\newblock The {EU} {AI} {A}ct, 2026.

\bibitem{Binns2021}
Reuben Binns and Michael Veale.
\newblock Is that your final decision? {M}ulti-stage profiling, selective
  effects, and {A}rticle 22 of the {GDPR}.
\newblock {\em International Data Privacy Law}, 11:319--332, 12 2021.

\bibitem{lundberg_unified_2017}
Scott Lundberg and Su-In Lee.
\newblock A {Unified} {Approach} to {Interpreting} {Model} {Predictions},
  November 2017.
\newblock arXiv:1705.07874 [cs].

\bibitem{ribeiro_why_2016}
Marco~Tulio Ribeiro, Sameer Singh, and Carlos Guestrin.
\newblock "{Why} {Should} {I} {Trust} {You}?": {Explaining} the {Predictions}
  of {Any} {Classifier}, August 2016.
\newblock arXiv:1602.04938 [cs].

\bibitem{sundararajan_axiomatic_2017}
Mukund Sundararajan, Ankur Taly, and Qiqi Yan.
\newblock Axiomatic {Attribution} for {Deep} {Networks}, June 2017.
\newblock arXiv:1703.01365 [cs].

\bibitem{Saranya2023}
A.~Saranya and R.~Subhashini.
\newblock A systematic review of explainable artificial intelligence models and
  applications: Recent developments and future trends.
\newblock {\em Decision Analytics Journal}, 7:100230, 6 2023.

\bibitem{Cunningham2023}
Hoagy Cunningham, Aidan Ewart, Logan Riggs, Robert Huben, and Lee Sharkey.
\newblock Sparse autoencoders find highly interpretable features in language
  models.
\newblock {\em 12th International Conference on Learning Representations, ICLR
  2024}, 10 2023.

\bibitem{Elhage2022}
Nelson Elhage, Tristan Hume, Catherine Olsson, Nicholas Schiefer, Tom Henighan,
  Shauna Kravec, Zac Hatfield-Dodds, Robert Lasenby, Dawn Drain, Carol Chen,
  Roger Grosse, Sam McCandlish, Jared Kaplan, Dario Amodei, Martin Wattenberg,
  and Christopher Olah.
\newblock Toy models of superposition.
\newblock 9 2022.

\bibitem{bau2017network}
David Bau, Bolei Zhou, Aditya Khosla, Aude Oliva, and Antonio Torralba.
\newblock Network dissection: Quantifying interpretability of deep visual
  representations.
\newblock In {\em Proceedings of the IEEE conference on computer vision and
  pattern recognition}, pages 6541--6549, 2017.

\bibitem{Koh2020}
Pang~Wei Koh, Thao Nguye, Yew~Siang Tang, Stephen Mussmann, Emma Pierso, Been
  Kim, and Percy Liang.
\newblock Concept bottleneck models.
\newblock {\em 37th International Conference on Machine Learning, ICML 2020},
  PartF168147-7:5294--5304, 7 2020.

\bibitem{Zarlenga2022}
Mateo~Espinosa Zarlenga, Pietro Barbiero, Gabriele Ciravegna, Giuseppe Marra,
  Francesco Giannini, Michelangelo Diligenti, Zohreh Shams, Frederic Precioso,
  Stefano Melacci, Adrian Weller, Pietro Lio, and Mateja Jamnik.
\newblock Concept embedding models: Beyond the accuracy-explainability
  trade-off.
\newblock {\em Advances in Neural Information Processing Systems}, 35, 9 2022.

\bibitem{Debot2025}
David Debot and Giuseppe Marra.
\newblock Quantifying the accuracy-interpretability trade-off in concept-based
  sidechannel models.
\newblock 10 2025.

\bibitem{Mahinpei2021PromisesAP}
Anita Mahinpei, Justin Clark, Isaac Lage, Finale Doshi-Velez, and Weiwei Pan.
\newblock Promises and pitfalls of black-box concept learning models.
\newblock {\em CoRR}, abs/2106.13314, 2021.

\bibitem{Havasi2022}
Marton Havasi, Sonali Parbhoo, and Finale Doshi-Velez.
\newblock Addressing leakage in concept bottleneck models, 10 2022.

\bibitem{Sawada2022ConceptBM}
Yoshihide Sawada and Keigo Nakamura.
\newblock Concept bottleneck model with additional unsupervised concepts.
\newblock {\em IEEE Access}, 10:41758--41765, 2022.

\bibitem{yuksekgonul2023posthoc}
Mert Yuksekgonul, Maggie Wang, and James Zou.
\newblock Post-hoc concept bottleneck models.
\newblock In {\em The Eleventh International Conference on Learning
  Representations}, 2023.

\bibitem{ismail2024concept}
Aya~Abdelsalam Ismail, Julius Adebayo, Hector~Corrada Bravo, Stephen Ra, and
  Kyunghyun Cho.
\newblock Concept bottleneck generative models.
\newblock In {\em The Twelfth International Conference on Learning
  Representations}, 2024.

\bibitem{Rocchi--Henry2025}
Alexandre Rocchi-Henry, Thomas Fel, and Gianni Franchi.
\newblock A geometric unification of concept learning with concept cones.
\newblock 12 2025.

\bibitem{SpinGlassTheoryAndBeyond}
M~Mezard, G~Parisi, and M~Virasoro.
\newblock {\em Spin Glass Theory and Beyond}.
\newblock WORLD SCIENTIFIC, 1986.

\bibitem{misra2023Sarcasm}
Rishabh Misra and Prahal Arora.
\newblock Sarcasm detection using news headlines dataset.
\newblock {\em AI Open}, 4:13--18, 2023.

\bibitem{he2021debertav3}
Pengcheng He, Jianfeng Gao, and Weizhu Chen.
\newblock Debertav3: Improving deberta using electra-style pre-training with
  gradient-disentangled embedding sharing, 2021.

\bibitem{barbieri-etal-2020-tweeteval}
Francesco Barbieri, Jose Camacho-Collados, Luis Espinosa~Anke, and Leonardo
  Neves.
\newblock {T}weet{E}val: Unified benchmark and comparative evaluation for tweet
  classification.
\newblock In {\em Findings of the Association for Computational Linguistics:
  EMNLP 2020}, pages 1644--1650, Online, November 2020. Association for
  Computational Linguistics.

\bibitem{wah2011caltech}
Catherine Wah, Steve Branson, Peter Welinder, Pietro Perona, and Serge
  Belongie.
\newblock The caltech-ucsd birds-200-2011 dataset.
\newblock 2011.

\bibitem{Radford2021}
Alec Radford, Jong~Wook Kim, Chris Hallacy, Aditya Ramesh, Gabriel Goh,
  Sandhini Agarwal, Girish Sastry, Amanda Askell, Pamela Mishkin, Jack Clark,
  Gretchen Krueger, and Ilya Sutskever.
\newblock Learning transferable visual models from natural language
  supervision.
\newblock {\em Proceedings of Machine Learning Research}, 139:8748--8763, 2
  2021.

\bibitem{Gershman2012}
Samuel~J. Gershman and Yael Niv.
\newblock Exploring a latent cause theory of classical conditioning.
\newblock {\em Learning \& Behavior 2012 40:3}, 40:255--268, 8 2012.

\bibitem{Schneider2012}
Michael Schneider, Xenia Vamvakoussi, and Wim~Van Dooren.
\newblock Conceptual change.
\newblock {\em Encyclopedia of the Sciences of Learning}, pages 735--738, 1
  2012.

\bibitem{Banerji2025}
Christopher R.~S. Banerji, Tapabrata Chakraborti, Aya~Abdelsalam Ismail,
  Florian Ostmann, and Ben~D. MacArthur.
\newblock Train clinical {AI} to reason like a team of doctors.
\newblock {\em Nature 2025 639:8053}, 639:32--34, 3 2025.

\bibitem{Marconato23}
Emanuele Marconato, Andrea Passerini, and Stefano Teso.
\newblock Interpretability is in the mind of the beholder: A causal framework
  for human-interpretable representation learning.
\newblock {\em Entropy}, 25(12), 2023.

\bibitem{higgins2017betavae}
Irina Higgins, Loic Matthey, Arka Pal, Christopher Burgess, Xavier Glorot,
  Matthew Botvinick, Shakir Mohamed, and Alexander Lerchner.
\newblock beta-{VAE}: Learning basic visual concepts with a constrained
  variational framework.
\newblock In {\em International Conference on Learning Representations}, 2017.

\bibitem{Chen2020}
Zhi Chen, Yijie Bei, and Cynthia Rudin.
\newblock Concept whitening for interpretable image recognition.
\newblock {\em Nature Machine Intelligence}, 2(12):772--782, December 2020.

\bibitem{MENG2026108346}
Xiyu Meng, Yilong Lin, Yuhan Wu, and Lu~Ying.
\newblock Enhancing concept alignment with explanatory interactive disentangled
  representation learning.
\newblock {\em Neural Networks}, 196:108346, 2026.

\bibitem{kuhn2012structure}
Thomas~S. Kuhn.
\newblock {\em The Structure of Scientific Revolutions}.
\newblock University of Chicago Press, 50th edition, 2012.

\bibitem{choi2024adaptive}
Jihye Choi, Jayaram Raghuram, Yixuan Li, Suman Banerjee, and Somesh Jha.
\newblock Adaptive concept bottleneck for foundation models.
\newblock In {\em ICML 2024 Workshop on Foundation Models in the Wild}, 2024.

\bibitem{bhalla2024interpreting}
Usha Bhalla, Alex Oesterling, Suraj Srinivas, Flavio Calmon, and Himabindu
  Lakkaraju.
\newblock Interpreting {CLIP} with sparse linear concept embeddings (spli{CE}).
\newblock In {\em The Thirty-eighth Annual Conference on Neural Information
  Processing Systems}, 2024.

\bibitem{linearrepresentationhyp}
Kiho Park, Yo~Joong Choe, and Victor Veitch.
\newblock The linear representation hypothesis and the geometry of large
  language models.
\newblock In Ruslan Salakhutdinov, Zico Kolter, Katherine Heller, Adrian
  Weller, Nuria Oliver, Jonathan Scarlett, and Felix Berkenkamp, editors, {\em
  Proceedings of the 41st International Conference on Machine Learning}, volume
  235 of {\em Proceedings of Machine Learning Research}, pages 39643--39666.
  PMLR, 21--27 Jul 2024.

\bibitem{kingma2017adammethodstochasticoptimization}
Diederik~P. Kingma and Jimmy Ba.
\newblock Adam: A method for stochastic optimization, 2017.

\end{thebibliography}
\bibliographystyle{unsrt}

\end{document}